\documentclass[twoside,11pt]{article}

\usepackage{jair,theapa}
\jairheading{30}{2007}{321--359}{11/06}{10/07}
\ShortHeadings{New Inference Rules for Max-SAT}{Li, Many\`a \& Planes}
\firstpageno{321}

\usepackage{amssymb,amsmath,latexsym}
\usepackage{algorithm,algorithmic}
\usepackage[dvips]{graphicx}
\usepackage{pstricks,pst-node}
\usepackage{rotating}
\usepackage{here}

\newtheorem{theoreme}{Rule}
\newtheorem{example}{Example}
\newtheorem{proof}{Proof}
\newtheorem{lemma}{Lemma}

\newcommand{\nRCBG}{\begin{tabular}{p{2.8in}}   \end{tabular} }
\newcommand{\Cqfd}{ \nRCBG   \textrm{$\blacksquare$} \ } 

\renewcommand{\algorithmicrequire}{\textbf{Input:}}
\renewcommand{\algorithmicensure}{\textbf{Output:}}

\title{New Inference Rules for Max-SAT}

\author{\name Chu Min Li \email chu-min.li@u-picardie.fr \\
\addr  LaRIA, Universit\'e de Picardie Jules Verne\\
33 Rue St. Leu, 80039 Amiens Cedex 01, France \\
\name Felip Many\`a \email felip@iiia.csic.es \\
\addr  IIIA, Artificial Intelligence Research Institute\\
CSIC, Spanish National Research Council\\
Campus UAB, 08193 Bellaterra, Spain \\
\name Jordi Planes \email jplanes@diei.udl.es \\
\addr Computer Science Department, Universitat de Lleida  \\
Jaume II, 69, 25001 Lleida, Spain  
}

\begin{document}

\maketitle

\begin{abstract}

Exact Max-SAT solvers, compared with SAT solvers, apply little inference at each node of the proof tree.
Commonly used SAT inference rules like unit propagation produce a simplified formula that preserves satisfiability but,
unfortunately, solving the Max-SAT problem for the simplified formula is not equivalent to solving it for the original formula.
In this paper, we define a number of original inference rules 
that, besides being applied efficiently, transform Max-SAT instances into equivalent Max-SAT
instances which are easier to solve.
The soundness of the rules, that can be seen as refinements of unit resolution adapted to Max-SAT, are proved
in a novel and simple way via an integer programming transformation. With the aim of
finding out how powerful the inference rules are in practice, we have developed a new Max-SAT solver, called
MaxSatz, which incorporates those rules, and performed  an experimental investigation. The results
provide empirical evidence that MaxSatz is very competitive, at least,
on random Max-2SAT, random Max-3SAT, Max-Cut, and Graph 3-coloring instances, as well as on the
benchmarks from the Max-SAT Evaluation 2006.
\end{abstract}

\section{Introduction}

In recent years there has been a growing interest in developing fast exact Max-SAT 
solvers~\cite{AGN98,AMP03,AMP05,GLMS03,LMP05,XZ04,ZSM03a} 
due to their potential  to solve over-constrained NP-hard problems encoded in the formalism of Boolean CNF formulas.
Nowadays, Max-SAT solvers are able to solve a lot of instances that are beyond the reach of the solvers developed just five years
ago. Nevertheless, there is yet a considerable gap between the difficulty of the instances solved with current SAT solvers and 
the instances solved with the best performing Max-SAT solvers.

The motivation behind our work is to bridge that gap between complete SAT solvers and exact Max-SAT solvers by investigating
how the technology previously developed for SAT~\cite{GN02,Li99,MarquesSakallah99,Zhang97,ZMMM01}
can be extended and incorporated into Max-SAT. More precisely, we focus the attention on
branch and bound Max-SAT solvers based on the Davis-Putnam-Logemann-Loveland (DPLL) procedure~\cite{DLL62,DavisPutnam60}.

One of the main differences between  SAT solvers and Max-SAT solvers
is that the former make an intensive use of unit propagation at each node of the proof tree.
Unit propagation, which is a highly powerful inference rule, transforms a SAT instance $\phi$ into a 
satisfiability equivalent SAT instance $\phi'$
which is easier to solve. Unfortunately, solving the Max-SAT problem for $\phi$ is,
in general, not {\it equivalent} to solving it for $\phi'$; i.e., the number of unsatisfied clauses
in $\phi$ and $\phi'$ is not the same for every truth assignment.
 For example, if we apply unit propagation to the CNF formula
$\phi=\{x_1, \bar x_1  \vee x_2, \bar x_1 \vee \neg x_2,\bar x_1  \vee x_3, \bar x_1 \vee \neg x_3 \}$,
we obtain $\phi'=\{ \Box,\Box\}$, but $\phi$ and $\phi'$ are not equivalent because any 
interpretation satisfying $\neg x_1$ unsatisfies one clause of $\phi$ and two clauses of $\phi'$. Therefore, 
if we want to compute an optimal solution, we cannot apply unit propagation as in SAT solvers.

We proposed in a previous work~\cite{LMP05} to use unit propagation to compute lower bounds in branch and bound Max-SAT solvers
instead of using unit propagation to simplify CNF formulas.
In our approach, we detect disjoint inconsistent subsets of clauses via unit propagation.
It turns out that the number of disjoint inconsistent subsets detected is an underestimation of the number of clauses
that will become unsatisfied when the current partial assignment is extended to a complete assignment.
That underestimation plus the number of clauses unsatisfied by the current partial assignment provides a good performing
lower bound, which captures the lower bounds based on inconsistency
counts that  most of the state-of-the-art Max-SAT solvers implement~\cite{AMP03a,AMP03,BF99,WF96,ZSM03a}, as well as 
other improved lower bounds~\cite{AMP04,AMP05,XZ04,XZ05}.

On the one hand, the number of disjoint inconsistent subsets detected is just a conservative underestimation for
the lower bound, since every inconsistent subset $\phi$ increases the lower bound by one independently of the number of clauses 
of $\phi$ unsatisfied by an optimal assignment. However, an optimal assignment can violate more than one clause of an inconsistent
subset. Therefore, we should be able to improve the lower bound based on counting the
number of disjoint inconsistent subsets of clauses.

On the other hand, despite the fact that good quality lower bounds prune large parts of the search space and accelerate dramatically the search
for an optimal solution, whenever the lower bound does not reach the best solution found so far (upper bound),
the solver continues exploring the search space below the current node.
During that search, solvers often redetect the same inconsistencies when computing 
the lower bound at different nodes. Basically, the problem with lower bound computation methods is that they do not simplify the CNF
formula in such a way that the unsatisfied clauses become explicit.
Lower bounds are just a pruning technique.

To overcome the above two problems, we define a set of {\em sound} inference rules that transform a 
Max-SAT instance  $\phi$ into a Max-SAT instance $\phi'$ which is easier to solve. In Max-SAT, an inference rule is
sound whenever $\phi$ and $\phi$' are equivalent.

Let us see an example of inference rule: Given a Max-SAT instance~$\phi$ that contains
three clauses of the form $l_1,l_2, \bar l_1 \vee \bar l_2$, where $l_1,l_2$ are literals,
we replace $\phi$ with the CNF formula
$$\phi'=(\phi - \{l_1,l_2, \bar l_1 \vee \bar l_2\}) \cup \{\Box,l_1 \vee l_2\}.$$
Note that the rule detects a contradiction from $l_1,l_2, \bar l_1 \vee \bar l_2$ and, therefore, replaces these
clauses with an empty clause. In addition, the rule adds the clause $l_1 \vee l_2$ to ensure the equivalence between $\phi$ and $\phi'$.
For any assignment containing either $l_1=0,l_2=1$, or $l_1=1,l_2=0$, or $l_1=1,l_2=1$, the number of unsatisfied clauses
in $\{l_1,l_2, \bar l_1 \vee \bar l_2\}$ is~1, but for any assignment containing $l_1=0,l_2=0$, the number of unsatisfied clauses
is~2. Note that even when any assignment containing $l_1=0,l_2=0$ is not the best assignment for the subset 
\{$l_1,l_2, \bar l_1 \vee \bar l_2$\}, it can be the best for the whole formula. By adding $l_1 \vee l_2$, the rule ensures that the number of unsatisfied clauses in $\phi$ and $\phi'$ is also the same
when $l_1=0,l_2=0$.

That inference rule adds the new clause $l_1 \vee l_2$,  which may contribute to another contradiction
detectable via unit propagation.
In this case, the rule allows to increase the lower bound by 2 instead of 1.
Moreover, the rule makes explicit a contradiction among $l_1,l_2, \bar l_1 \vee \bar l_2$, so that the contradiction 
does not need to be redetected below the current node.

Some of the inference rules defined in the paper are already known in the literature~\cite{BR99,NR00}, others are
original for Max-SAT. The new rules were inspired by different unit resolution refinements applied in SAT,
and were selected because they could be applied in a natural and efficient way. In a sense,
we can summarize our work telling that we have defined the Max-SAT counterpart of SAT unit propagation.

With the aim of finding out how powerful the inference rules are in practice, we have designed and implemented a new Max-SAT solver, called
MaxSatz, which incorporates those rules, as well as the lower bound defined in a previous work~\cite{LMP05}, and performed  an experimental investigation. 
The results
provide empirical evidence that MaxSatz is very competitive, at least,
on random Max-2SAT, random Max-3SAT, Max-Cut, and Graph 3-coloring instances, as well as on the
benchmarks from the Max-SAT Evaluation 2006\footnote{http://www.iiia.csic.es/\~{}maxsat06}.

The structure of the paper is as follows.  In Section~\ref{preliminaries}, we give some preliminary definitions. 
In Section~\ref{basic-solver}, we describe a basic branch and bound Max-SAT solver.
In Section~\ref{inference-rules}, we define the inference rules and prove their soundness in a novel and simple way
via an integer programming transformation. We also give examples to illustrate that the inference rules may produce 
better quality lower bounds.
In Section \ref{implementation}, we present the implementation of the inference rules in MaxSatz.
 In Section~\ref{MaxSatz}, we describe the main features of MaxSatz.
In Section~\ref{Experiments}, we report on the experimental investigation.
In Section~\ref{related-work}, we present the related work.
In Section~\ref{conclusions}, we present the conclusions and future work.

\section{Preliminaries}\label{preliminaries}

In propositional logic a variable $x_i$ may take values 0 (for false) or 1 (for true). A literal $l_i$ is a variable
$x_i$ or its negation $\bar x_i$. A clause is  a disjunction of  literals, and a CNF formula $\phi$ is a conjunction of clauses.  
The length of a clause is the number of its literals. The size of
$\phi$, denoted by $|\phi|$, is the sum of the length of all its clauses.

An assignment of truth values to the propositional variables satisfies a literal $x_i$ if $x_i$ takes the value~1
and satisfies a literal $\bar x_i$ if $x_i$ takes the value~0, satisfies a clause if it satisfies at least one
literal of the clause, and satisfies a CNF formula if it satisfies all the clauses of the formula. An empty clause, 
denoted by $\Box$, contains no literals and cannot be satisfied. An assignment for a CNF formula $\phi$ is complete if all the variables
occurring in $\phi$ have been assigned; otherwise, it is partial.

The  Max-SAT problem for a CNF formula $\phi$ is the problem of finding an assignment of values to propositional
variables that minimizes the number of unsatisfied clauses (or equivalently, that maximizes 
the number of satisfied clauses). Max-SAT is called  Max-$k$SAT when all the clauses have $k$ literals per clause. 
In the following, we represent a CNF formula as a multiset of clauses, since duplicated clauses are allowed in a Max-SAT instance.

CNF formulas $\phi_1$ and $\phi_2$ are equivalent if $\phi_1$ and $\phi_2$ have the same number of unsatisfied clauses
for every complete assignment of $\phi_1$ and $\phi_2$.

\section{A Basic Max-SAT Solver}\label{basic-solver}

The space of all possible assignments for a CNF formula $\phi$ can be represented as a search tree,
where internal nodes represent partial assignments and leaf nodes represent complete assignments.
A basic branch and bound algorithm for Max-SAT explores the search tree in a depth-first manner.
At every node, the algorithm compares the number of clauses unsatisfied by the best 
complete assignment found so far ---called upper bound ($UB$)--- with the number of clauses
unsatisfied by the current partial assignment ($\#emptyClauses$) plus an underestimation of the minimum number 
of non-empty clauses that will become unsatisfied if we extend the current partial assignment into a complete assignment ($underestimation$).

The sum $\#emptyClauses + underestimation$  is a lower bound ($LB$) of the minimum number of clauses  unsatisfied by 
any complete assignment extended from the current partial assignment. Obviously,
 if $LB \geq UB$, a better solution cannot be found from this point in search.
In that case, the algorithm prunes the subtree below
the current node and backtracks to a higher level in the search tree. 

If~$LB < UB$, the algorithm tries to find a possible better solution by extending the current partial assignment
by instantiating one more variable; which leads to the creation of  two branches from the current branch: the
left branch corresponds to assigning the new variable to false, and the right branch
corresponds to assigning the new variable to true.
In that case, the formula associated with the left (right) branch is obtained from the formula of
the current node by deleting all the clauses containing the literal $\bar x$ ($x$) and 
removing all the occurrences of the literal $x$ ($\bar x$); i.e., the algorithm applies the 
\emph{\mbox{one-literal} rule}. 

The solution to Max-SAT is the value that $UB$ takes after
exploring the entire search tree.

\begin{figure} 
\begin{algorithmic}[1]
\REQUIRE{\emph{max-sat}($\phi$, $UB$)} : A CNF formula $\phi$ and an upper bound $UB$ 

\STATE $\phi \leftarrow \emph{simplifyFormula}(\phi)$;
\IF{$\phi=\emptyset$ or $\phi$ only contains empty clauses}
\STATE return $\#emptyClauses(\phi)$;
\ENDIF
\STATE $LB \leftarrow \#emptyClauses(\phi)+underestimation(\phi, UB)$;
\IF{ $LB \ge UB$}
\STATE return $UB$;
\ENDIF
\STATE $x \leftarrow \emph{selectVariable}(\phi)$;
\STATE $UB \leftarrow \min( UB, \emph{max-sat}(\phi_{\bar x}, UB))$;
\STATE return $\min(UB, \emph{max-sat}(\phi_x, UB))$;

\ENSURE The minimal number of unsatisfied clauses of $\phi$  
\end{algorithmic}
\caption{A basic branch and bound algorithm for Max-SAT}
\label{algoBnB}
\end{figure}

Figure~\ref{algoBnB} shows the pseudo-code of a basic solver for Max-SAT.
We use the following notations:

\begin{itemize}

\item $\emph{simplifyFormula}(\phi)$ is a procedure that simplifies $\phi$ by applying 
sound inference rules.

\item $\#emptyClauses(\phi)$ is a function that returns the
number of empty clauses in~$\phi$.

\item $LB$ is a lower bound of the minimum number of unsatisfied clauses
in $\phi$ if the current partial assignment is extended to a complete assignment.
 We assume that its initial value is~0.

\item $\emph{underestimation}(\phi,UB)$ is a function that returns an underestimation
of the minimum number of non-empty clauses in $\phi$ that will become unsatisfied if the current partial
assignment is extended to a complete assignment.

\item $UB$ is an upper bound of the number of unsatisfied clauses in an optimal
solution. We assume that its initial value is the total number of clauses in the input formula.
 
\item $\emph{selectVariable}(\phi)$ is a function that returns a variable of $\phi$
following an heuristic.

\item $\phi_x$ ($\phi_{\bar x}$) is the formula obtained by applying the one-literal rule
to $\phi$ using the literal $x$ ($\bar x$).
\end{itemize}

State-of-the-art Max-SAT solvers implement the basic algorithm augmented with 
powerful inference techniques, good quality lower bounds, clever variable selection
heuristics, and efficient data structures.

We have recently defined~\cite{LMP05} a lower bound computation method in which the underestimation
of the lower bound is the number of disjoint inconsistent subsets that can be detected using unit propagation. The
pseudo-code is shown in Figure~\ref{estimate}.

\begin{figure} 
\begin{algorithmic}[1]
\REQUIRE{\emph{underestimation}($\phi$, $UB$)}  : A CNF formula $\phi$ and an upper bound $UB$ 

\STATE  $underestimation \leftarrow 0$;

\STATE apply the one-literal rule to the unit clauses of $\phi$ (unit propagation) until an empty clause is derived;

\IF{no empty clause can be derived}
\STATE return $underestimation$;
\ENDIF
\STATE $\phi \leftarrow \phi$ without the clauses that have been used to derive the empty clause;

\STATE  $underestimation:= underestimation +1$;
\IF{$underestimation$+\#$emptyClauses$($\phi$) $\ge$ $UB$}
\STATE return $underestimation$;
\ENDIF
\STATE go to 2;

\ENSURE the underestimation of the lower bound for $\phi$
\end{algorithmic}
\caption{Computation of the underestimation using unit propagation}\label{estimate}
\end{figure}

\begin{example}\label{formula-lb}
Let $\phi$ be the following CNF formula:
$$\{x_1, x_2, x_3, x_4, \bar x_1 \vee \bar x_2 \vee \bar x_3, \bar x_4,
x_5, \bar x_5 \vee \bar x_2,  \bar x_5 \vee x_2\}.$$

With our approach we are able to establish that the number of disjoint inconsistent subsets of clauses 
in $\phi$ is at least~3. Therefore, the underestimation of the lower bound is~3.
The steps performed are the following ones:

\begin{enumerate}

\item  $\phi= \{x_4,  \bar x_4, x_5, \bar x_5 \vee \bar x_2,  \bar x_5 \vee x_2\}$,
the first inconsistent subset detected using unit propagation is $\{x_1, x_2, x_3,\bar x_1 \vee \bar x_2 \vee \bar x_3\}$, and
$underestimation=1$.

\item $\phi= \{x_5, \bar x_5 \vee \bar x_2,  \bar x_5 \vee x_2\}$, the second
inconsistent subset detected using unit propagation is  $\{x_4,\bar x_4\}$, and $underestimation=2$.

\item $\phi= \emptyset$, the third inconsistent subset detected using unit propagation is
$\{x_5, \bar x_5 \vee \bar x_2, \bar x_5 \vee x_2\}$, and $underestimation=3$. Since $\phi$ is empty,
the algorithm stops.
\end{enumerate}

\end{example}

\section{Inference Rules}\label{inference-rules}

We define the set of inference rules considered in the paper.
They were inspired by different unit resolution refinements applied in SAT,
and were selected because they could be applied in a natural and efficient way. 
Some of them are already known in the literature~\cite{BR99,NR00}, others are
original for Max-SAT.

Before presenting the rules, we define an integer programming transformation of a CNF formula used to establish the
soundness of the rules.  The method of proving soundness is novel in Max-SAT, and provides clear
and short proofs.

\subsection{Integer Programming Transformation of a CNF Formula}

Assume that $\phi=\{c_1, ..., c_m\}$ is a CNF formula  with $m$ clauses over the 
variables $x_1, ..., x_n$. Let $c_i$ $(1 \leq i \leq m)$ be
$x_{i_1} \vee ... \vee x_{i_k} \vee \bar{x}_{i_{k+1}} \vee ... \vee \bar{x}_{i_{k+r}}$. Note that
we put all positive literals in $c_i$ before the negative ones.

We consider all the variables in $c_i$ as integer variables taking values~0 or~1, and
define the integer transformation of $c_i$ as
$${\cal E}_i(x_{i_1},...,x_{i_k},x_{i_{k+1}},...,x_{i_{k+r}})
=(1-x_{i_1})...(1-x_{i_k})x_{i_{k+1}}...x_{i_{k+r}}$$

Obviously, ${\cal E}_i$ has value 0 iff at least one of the variables $x_{i_j}$'s ($1 \leq j \leq k$) is instantiated to~1 or
at least one of the variables $x_{i_s}$'s ($k+1 \leq s \leq k+r$) is instantiated to~0. In other words, ${\cal E}_i$=0
iff $c_i$ is satisfied. Otherwise ${\cal E}_i$=1.

A literal $l$  corresponds to an integer denoted by $l$ itself for our 
convenience. The intention of the correspondence is that the literal $l$ is satisfied if the integer 
$l$ is 1 and is unsatisfied if the integer $l$ is 0.
So if $l$ is a positive literal $x$, the corresponding integer $l$ is $x$, $\bar l$ is 1-$x$=1-$l$, 
and if $l$ is a negative literal $\bar x$, $l$ is 1-$x$ and $\bar l$ is $x$=1-(1-$x$)=1-$l$. 
Consequently, $\bar l$=1-$l$ in any case. 

We now generically write $c_i$
as $l_1 \vee l_2 \vee ... \vee l_{k+r}$. Its integer programming transformation is
$${\cal E}_i=(1-l_1)(1-l_2)...(1-l_{k+r}).$$

The integer programming transformation of a CNF formula  $\phi=\{c_1, ..., c_m\}$ over the
variables $x_1, ..., x_n$ is defined as
\begin{equation} \label{Energy}
{\cal E}(x_1, ..., x_n)=\sum_{i=1}^m {\cal E}_i
\end{equation}

That integer programming transformation was used~\cite{HC97,LH05} to design a local search procedure, and is called pseudo-Boolean formulation by~\citeauthor{BH02}~\citeyear{BH02}. Here, we extend it to empty clauses:
if $c_i$ is empty, then ${\cal E}_i$=1. 

Given an assignment $A$ over the variables $x_1, ..., x_n$, the value of ${\cal E}$ is 
the number of unsatisfied clauses in 
$\phi$. If $A$ satisfies all clauses in $\phi$, then ${\cal E}=0$. 
Obviously, the minimum number of unsatisfied clauses of $\phi$ 
is the minimum value of ${\cal E}$. 

Let ${\phi}_1$ and ${\phi}_2$ be two CNF formulas, and let ${\cal E}_1$ and ${\cal E}_2$ be their
integer programming transformations. It is clear that ${\phi}_1$ and ${\phi}_2$ are equivalent if, and 
only if, ${\cal E}_1$=${\cal E}_2$ for every complete assignment for ${\phi}_1$ and ${\phi}_2$.

\subsection{Inference Rules}

We next define the inference rules and prove their soundness using the previous integer programming transformation.
In the rest of the section, ${\phi}_1$, ${\phi}_2$ and ${\phi}'$ denote CNF formulas, and
 ${\cal E}_1$, ${\cal E}_2$, and ${\cal E}'$ their integer programming transformations.
To prove that $\phi_1$ and $\phi_2$ are equivalent, we prove 
that ${\cal E}_1={\cal E}_2$.

\begin{theoreme}  \label{resolution} \cite{BR99}
 If ${\phi}_1$=\{$l_1 \vee l_2 \vee ... \vee l_k,~{\bar l}_1 \vee l_2  \vee ... \vee l_k\} \cup {\phi}'$,
 then ${\phi}_2$=\{$l_2  \vee ... \vee l_k\} \cup {\phi}'$ is equivalent to ${\phi}_1$.
\end{theoreme}
\begin{proof}
 \begin{eqnarray} 
{\cal E}_1 & = & (1-l_1)(1-l_2)...(1-l_k)+l_1(1-l_2)...(1-l_k)+{\cal E}' \nonumber\\ 
           & = & (1-l_2)...(1-l_k)+{\cal E}' \nonumber\\ 
           & = & {\cal E}_2 \nonumber ~~~~\Cqfd
\end{eqnarray}
\end{proof}
General case resolution does not work in  Max-SAT~\cite{BR99}. Rule \ref{resolution} establishes that  resolution works 
when two clauses give a strictly shorter resolvent.

 Rule \ref{resolution}  is known in the literature as {\em replacement of almost common
clauses}. We pay special attention to the case k=2, where the resolvent is a unit clause, and to the case k=1, 
where the resolvent is the empty clause. We  describe this latter case in the following rule:

\begin{theoreme}  \label{linear1} \cite{NR00}
If ${\phi}_1$=\{$l,~\bar l\} \cup {\phi}'$, then ${\phi}_2$=\{$\Box\} \cup \phi'$ is equivalent to ${\phi}_1$.
\end{theoreme}
\begin{proof}
${\cal E}_1$=1-$l$+ $l$+${\cal E}'$=1+ ${\cal E}'$=$ {\cal E}_2$ ~~~~\Cqfd
\end{proof}
Rule \ref{linear1}, which is known as {\em complementary unit clause rule}, can be used to replace two
complementary unit clauses with an empty clause.  The new empty clause contributes to the lower bounds of the
search space below the current node  by incrementing the number of unsatisfied clauses, but not by incrementing the underestimation, which means that this contradiction does not have to be redetected again. In practice, that simple rule gives rise to
considerable gains.

The following rule is a more complicated case:

\begin{theoreme}  \label{linear2} 
If ${\phi}_1$=\{$l_1, ~\bar{l}_1 \vee \bar{l}_2, ~l_2\} \cup \phi'$, then ${\phi}_2$=\{$\Box,~ l_1 \vee l_2\} \cup \phi'$ is equivalent to ${\phi}_1$.
\end{theoreme}
\begin{proof}
\begin{eqnarray} 
{\cal E}_1 & = & 1-l_1+l_1l_2+1-l_2+{\cal E}' \nonumber\\ 
           & = & 1+1-l_1+l_2(l_1-1)+{\cal E}' \nonumber\\ 
           & = & 1+1-l_1-l_2(1-l_1)+{\cal E}' \nonumber\\ 
           & = & 1+(1-l_1)(1-l_2)+{\cal E}' \nonumber\\ 
           & = & {\cal E}_2 \nonumber ~~~~\Cqfd
\end{eqnarray}
\end{proof}

Rule \ref{linear2} replaces three clauses with an empty clause, and adds
a new binary clause to keep the equivalence between $\phi_1$ and $\phi_2$.

Pattern $\phi_1$ was considered to compute underestimations 
by~\citeauthor{AMP04}~\citeyear{AMP04} and \citeauthor{SZ04}~\citeyear{SZ04}; 
and is also captured by our method of computing underestimations based on unit propagation~\cite{LMP05}. 
\citeauthor{LH05a} mentioned \citeyear{LH05a} that existential directional arc consistency \cite{GZHL05} 
can capture this rule.
Note that underestimation computation methods by~\citeauthor{AMP04} and \citeauthor{SZ04} do not add any additional clause as in our approach, 
they just detect contradictions. 

Let us define a rule that generalizes Rule \ref{linear1} and Rule \ref{linear2}. Before presenting the rule, we
define a lemma needed to prove its soundness.

\begin{lemma} \label{lemma1} 
If ${\phi}_1$=\{$l_1,~ \bar{l}_1 \vee l_2\} \cup \phi'$ and ${\phi}_2$=\{$l_2, ~\bar{l}_2 \vee l_1\} \cup \phi'$, 
then ${\phi}_1$ and ${\phi}_2$ are 
equivalent.
\end{lemma}
\begin{proof}
\begin{eqnarray} 
 {\cal E}_1 & = & 1-l_1+l_1(1-l_2)+{\cal E}' \nonumber\\
            & = & 1-l_1+l_1-l_1l_2+{\cal E}' \nonumber\\
            & = & 1-l_2+l_2-l_1l_2+{\cal E}' \nonumber\\
            & = & 1-l_2+(1-l_1)l_2+{\cal E}' \nonumber\\
            & = & {\cal E}_2 \nonumber ~~~~\Cqfd
\end{eqnarray}
\end{proof}

\begin{theoreme} \label{linear3}
If ${\phi}_1$=\{$l_1, ~{\bar l}_1 \vee l_2, ~{\bar l}_2 \vee l_3,~ ...,~ {\bar l}_{k} \vee l_{k+1}, ~{\bar l}_{k+1}\} \cup \phi'$,
then ${\phi}_2$=\{$\Box, ~l_1 \vee {\bar l}_2, ~l_2 \vee {\bar l}_3, ~..., ~l_{k} \vee {\bar l}_{k+1}\} \cup \phi'$ is equivalent to
${\phi}_1$.
\end{theoreme}

\begin{proof}
We prove the soundness of the rule by induction on $k$. 
When $k$=1, ${\phi}_1=\{l_1,~ {\bar l}_1 \vee l_2,~ {\bar l}_{2}\} \cup \phi'$. By applying Rule~\ref{linear2}, we get
$\{\Box,~ l_1 \vee {\bar l}_{2}\} \cup \phi'$, which is ${\phi}_2$ when $k=1$. Therefore, ${\phi}_1$ and ${\phi}_2$
are equivalent.

Assume that Rule \ref{linear3} is sound for $k=n$. Let us prove that it is sound for $k=n+1$. In that case:
$${\phi}_1=\{l_1,~ {\bar l}_1 \vee l_2, ~{\bar l}_2 \vee l_3, ~...,~ {\bar l}_{n} \vee l_{n+1}, ~{\bar l}_{n+1} \vee l_{n+2}, ~{\bar l}_{n+2}\} \cup \phi'.$$

\noindent  By applying Lemma \ref{lemma1} to the last two clauses of ${\phi}_1$ (before $\phi'$), we get 
$$\{l_1, ~{\bar l}_1 \vee l_2, ~{\bar l}_2 \vee l_3,~ ...,~ {\bar l}_{n} \vee l_{n+1},~ {\bar l}_{n+1}, ~l_{n+1} \vee {\bar l}_{n+2}\} \cup \phi'.$$

\noindent  By applying the induction hypothesis to the first $n+1$ clauses of the previous CNF formula, we get
$$\{\Box, ~l_1 \vee {\bar l}_2, ~l_2 \vee {\bar l}_3, ~..., ~l_{n} \vee {\bar l}_{n+1}, ~l_{n+1} \vee {\bar l}_{n+2} \} \cup \phi',$$

\noindent   which  is ${\phi}_2$ when $k=n+1$. Therefore, ${\phi}_1$ and ${\phi}_2$
are equivalent and the rule is sound. ~~~~\Cqfd
\end{proof}

Rule~\ref{linear3}  is an original inference rule. It captures linear unit resolution refutations in which clauses
and resolvents are used exactly once.
The rule simply adds an empty clause, eliminates two unit clauses and the binary clauses used in the refutation, and 
adds new binary clauses that are obtained by negating the literals of the eliminated binary clauses. So, all the operations involved can be performed efficiently.

 Rule \ref{linear2} and Rule \ref{linear3} make explicit a contradiction, which does not need to be redetected in the current subtree. So,
the lower bound computation becomes more incremental. Moreover, the binary clauses added by Rule~\ref{linear2} and Rule~\ref{linear3}
may contribute to compute better quality lower bounds either by acting as premises of an inference rule or by being part of
an inconsistent subset of clauses, as is illustrated in the following example.

\begin{example}
Let $\phi$=\{$x_1,  \bar x_1 \vee \bar x_2, x_3,  \bar x_3 \vee x_2, x_4, \bar x_1 \vee \bar x_4,  \bar x_3 \vee \bar x_4$\}. 
Depending on the ordering in which unit clauses are propagated, unit propagation detects one of the following three inconsistent subsets 
of clauses:  \{$x_1, \bar x_1 \vee \bar x_2, x_3, \bar x_3 \vee x_2$\},
\{$x_1, x_4, \bar x_1 \vee \bar x_4$\}, or  \{$x_3, x_4, \bar x_3 \vee \bar x_4$\}. Once an inconsistent subset is detected and removed, 
the remaining set of clauses is satisfiable. Without applying Rule \ref{linear2} and Rule \ref{linear3}, the 
lower bound computed is~1, because the underestimation computed using unit propagation is~1.
 
Note that Rule~\ref{linear3} can be applied to the first inconsistent subset \{$x_1, \bar x_1 \vee \bar x_2, x_3, \bar x_3 \vee x_2$\}.
If Rule~\ref{linear3} is applied, a contradiction is made explicit and the clauses $x_1 \vee x_2$ and $x_3 \vee \bar x_2$ are added.
So, $\phi$ becomes $\{\Box, x_1 \vee x_2, x_3 \vee \bar x_2, x_4, \bar x_1 \vee \bar x_4,  \bar x_3 \vee \bar x_4\}$. It turns out that
$\phi - \{\Box\}$ is an inconsistent set of clauses detectable by unit propagation. Therefore, the lower bound computed is~2.

If the inconsistent subset \{$x_1, x_4, \bar x_1 \vee \bar x_4$\} is detected, Rule \ref{linear2} can be applied.
Then,  a contradiction is made explicit and the clause  $x_1 \vee x_4$ is added.
So, $\phi$ becomes $\{\Box, x_1 \vee x_4, \bar x_1 \vee \bar x_2, x_3,  \bar x_3 \vee x_2,  \bar x_3 \vee \bar x_4\}$.
It turns out that $\phi - \{\Box\}$ is an inconsistent set of clauses detectable by unit propagation. 
Therefore, the lower bound computed is~2.

Similarly, if the inconsistent subset \{$x_3, x_4, \bar x_3 \vee \bar x_4$\} is detected and Rule \ref{linear2} 
is applied,  the lower bound computed is~2.

We observe that, in this example, Rule \ref{linear2} and Rule \ref{linear3}  not only make explicit a contradiction,
but also allow to improve the lower bound.
\end{example}

Unit propagation needs at least one unit clause to detect a contradiction. A drawback of Rule \ref{linear2} and Rule \ref{linear3} is that they consume two unit clauses for deriving just one contradiction. A possible situation 
is that, after branching, those two unit clauses could allow unit propagation to derive two disjoint inconsistent subsets of clauses, as we show in the following example.

\begin{example}
Let $\phi$=\{$x_1, \bar x_1 \vee x_2, \bar x_1 \vee x_3, \bar x_2 \vee \bar x_3 \vee x_4, x_5, \bar x_5 \vee x_6, \bar x_5 \vee x_7, \bar x_6 \vee \bar x_7 \vee x_4, \bar x_1 \vee \bar x_5$\}. Rule \ref{linear2} replaces $x_1$, $x_5$, and $\bar x_1 \vee \bar x_5$ with an empty clause and $x_1 \vee x_5$. After that, if $x_4$ is selected as the next branching variable and is assigned 0, 
there is no unit clause in $\phi$ and no contradiction can be detected via unit propagation. The lower bound is 1 in this
situation. However,
if Rule \ref{linear2} was not applied before branching, $\phi$ has two unit clauses after branching.
In this case, the propagation of $x_1$ allows to detect the inconsistent subset 
\{$x_1, \bar x_1 \vee x_2, \bar x_1 \vee x_3, \bar x_2 \vee \bar x_3$\}, and the propagation of $x_5$ allows to detect the
inconsistent subset  \{$x_5, \bar x_5 \vee x_6, \bar x_5 \vee x_7, \bar x_6 \vee \bar x_7$\}. So, the lower bound 
computed after branching is 2.
\end{example}

On the one hand, Rule \ref{linear2} and Rule \ref{linear3}  add clauses that can contribute  to detect additional conflicts.
On the other hand, each application of Rule \ref{linear2} and Rule \ref{linear3} consumes two unit clauses, 
which cannot be used again to detect further conflicts. 
The final effect of these two factors will be empirically analyzed in Section \ref{Experiments}.

Finally, we present two new rules that capture unit resolution refutations in which (i)~exactly one unit clause
is consumed, and (ii)~the unit clause is used twice in the linear derivation of the empty clause.

\begin{theoreme} \label{nonlinear1}
If ${\phi}_1$=\{$l_1, ~{\bar l}_1 \vee l_2,~ {\bar l}_1 \vee l_3, ~{\bar l}_2 \vee {\bar l}_3\} \cup {\phi}'$,
then ${\phi}_2$=\{$\Box, ~l_1 \vee {\bar l}_2 \vee {\bar l}_3, ~{\bar l}_1 \vee l_2 \vee l_3\} \cup  {\phi}'$ is equivalent to ${\phi}_1$.
\end{theoreme}

\begin{proof}
\begin{eqnarray} 
{\cal E}_1 & = & 1-l_1+l_1(1-l_2)+l_1(1-l_3)+l_2l_3+{\cal E}' \nonumber\\
           & = & 1-l_1+l_1-l_1l_2+l_1-l_1l_3+l_2l_3+{\cal E}' \nonumber\\
           & = & 1+l_2l_3-l_1l_2l_3+l_1-l_1l_2-l_1l_3+l_1l_2l_3+{\cal E}' \nonumber\\
           & = & 1+(1-l_1)l_2l_3+l_1(1-l_2-l_3+l_2l_3)+{\cal E}' \nonumber\\
           & = & 1+(1-l_1)l_2l_3+l_1(1-l_2)(1-l_3)+{\cal E}' \nonumber\\
           & = & {\cal E}_2 \nonumber ~~~~\Cqfd
\end{eqnarray}
\end{proof}

We can combine a linear derivation with Rule \ref{nonlinear1} to obtain Rule \ref{nonlinear2}:

\begin{theoreme} \label{nonlinear2}
If ${\phi}_1$=\{$l_1,~ {\bar l}_1 \vee l_2,~ {\bar l}_2 \vee l_3
,~ ..., ~{\bar l}_{k} \vee l_{k+1}
, ~{\bar l}_{k+1} \vee l_{k+2}, ~{\bar l}_{k+1} \vee l_{k+3}, ~{\bar l}_{k+2} \vee {\bar l}_{k+3}\} \cup {\phi}'$,
then ${\phi}_2$=\{$\Box, ~l_1 \vee {\bar l}_2, ~l_2 \vee {\bar l}_3,~ ..., ~l_{k} \vee 
{\bar l}_{k+1},~ l_{k+1} \vee {\bar l}_{k+2} \vee {\bar l}_{k+3}, ~{\bar l}_{k+1} \vee 
l_{k+2} \vee l_{k+3}\} \cup {\phi}'$ is equivalent to ${\phi}_1$.
\end{theoreme}

\begin{proof}
We prove the soundness of the rule by induction on $k$. When $k$=1,  
$${\phi}_1=\{l_1,~ {\bar l}_1 \vee l_2, ~{\bar l}_{2} \vee l_{3}, ~{\bar l}_{2} \vee l_{4}, ~{\bar l}_{3} \vee {\bar l}_{4}\} \cup {\phi}'.$$ 
\noindent By Lemma 1, we get
$$\{l_1 \vee {\bar l}_2, ~l_2 ,~ {\bar l}_2 \vee l_3,~ {\bar l}_2 \vee l_4, ~{\bar l}_3 \vee {\bar l}_4 \} \cup {\phi}'.$$
 
\noindent By Rule \ref{nonlinear1}, we get 
 $$ \{l_1 \vee {\bar l}_2, ~\Box,~ l_2 \vee {\bar l}_3 \vee {\bar l}_4, ~{\bar l}_2 \vee l_3 \vee l_4\} \cup {\phi}',$$
 
\noindent which  is ${\phi}_2$ when $k=1$. Therefore, ${\phi}_1$ and ${\phi}_2$ are equivalent.

Assume that Rule~\ref{nonlinear2} is sound for $k=n$. Let us prove that it is sound for $k=n+1$. In that case:
$$
{\phi}_1 = \{l_1, ~{\bar l}_1 \vee l_2, ~{\bar l}_2 \vee l_3,~ ..., ~{\bar l}_{n+1} \vee l_{n+2},~ {\bar l}_{n+2} \vee l_{n+3}, ~{\bar l}_{n+2} \vee l_{n+4}, ~{\bar l}_{n+3} \vee {\bar l}_{n+4}\} \cup {\phi}'.$$

\noindent By Lemma~\ref{lemma1}, we get
$$\{l_1 \vee {\bar l}_2, ~l_2, ~{\bar l}_2 \vee l_3,~ ...,~ {\bar l}_{n+1} \vee l_{n+2},~ {\bar l}_{n+2} \vee l_{n+3}, ~{\bar l}_{n+2} \vee l_{n+4}, ~{\bar l}_{n+3} \vee {\bar l}_{n+4}\} \cup {\phi}'.$$

\noindent By applying the induction hypothesis, we get
 $$\{l_1 \vee {\bar l}_2, ~\Box,~ l_2 \vee {\bar l}_3,~ ..., ~l_{n+1} \vee {\bar l}_{n+2},~ l_{n+2} \vee {\bar l}_{n+3} \vee {\bar l}_{n+4}, ~{\bar l}_{n+2} \vee l_{n+3} \vee l_{n+4}\} \cup {\phi}',$$

\noindent which  is ${\phi}_2$ when $k=n+1$. Therefore, ${\phi}_1$ and ${\phi}_2$
are equivalent and the rule is sound. ~~~~\Cqfd

\end{proof}

Similarly to Rule \ref{linear2} and Rule \ref{linear3}, Rule \ref{nonlinear1} and Rule \ref{nonlinear2} make explicit a contradiction,
which does not need to be redetected in subsequent search. Therefore, the lower bound computation becomes more incremental.
Moreover, they also add clauses which can improve the quality of the lower bound, as illustrated in the following example.

\begin{example}
Let $\phi$=\{$x_1, \bar x_1 \vee x_2, \bar x_1 \vee x_3, \bar x_2 \vee \bar x_3, x_4, x_1 \vee \bar x_4, \bar x_2 \vee \bar x_4, \bar x_3 \vee \bar x_4$\}. Depending on the ordering  in which unit clauses are propagated, unit propagation can detect one of the following inconsistent subsets: \{$x_1, \bar x_1 \vee x_2, \bar x_1 \vee x_3, \bar x_2 \vee \bar x_3$\}, \{$x_4, x_1 \vee \bar x_4, \bar x_2 \vee \bar x_4, \bar x_1 \vee x_2$\},
\{$x_4, x_1 \vee \bar x_4, \bar x_3 \vee \bar x_4, \bar x_1 \vee x_3$\}, in which Rule \ref{nonlinear1} is applicable. If Rule \ref{nonlinear1} is not applied, the lower bound computed using the $underestimation$ function of Figure \ref{estimate} is 1, since 
the remaining clauses of $\phi$ are satisfiable once the inconsistent subset of clauses is removed. Rule \ref{nonlinear1} allows to add two ternary clauses contributing to another contradiction. For example, Rule~\ref{nonlinear1} applied to \{$x_1, \bar x_1 \vee x_2, \bar x_1 \vee x_3, \bar x_2 \vee \bar x_3$\} adds to $\phi$ clauses $x_1 \vee \bar x_2 \vee \bar x_3$ and $\bar x_1 \vee  x_2 \vee  x_3$, which, with the remaining clauses of $\phi$ (\{$x_4, x_1 \vee \bar x_4, \bar x_2 \vee \bar x_4, \bar x_3 \vee \bar x_4$\}), give the second contradiction detectable via unit propagation. So the lower bound computed using Rule~\ref{nonlinear1} is~2.

\end{example}

In contrast to Rule \ref{linear2} and Rule \ref{linear3},  Rule \ref{nonlinear1} and Rule \ref{nonlinear2}  consume exactly one unit 
clause for deriving an empty clause. Since a unit clause can be used at most once to derive a conflict via unit propagation, 
Rule \ref{nonlinear1} and Rule \ref{nonlinear2} do not limit the detection of further conflicts via unit propagation.

\section{Implementation of Inference Rules}
\label{implementation}
In this section, we describe the implementation of all the inference rules presented in Section~\ref{inference-rules}.
We suppose that the CNF formula is loaded and, for every literal $\ell$, a list of clauses containing $\ell$ is constructed. 
The application of a rule means that some clauses in $\phi_1$ are removed from the CNF formula, new clauses in $\phi_2$ are inserted 
into the formula, and the lower bound is increased by 1. Note that in all the inference rules selected in our approach, $\phi_2$ 
contains fewer literals and fewer clauses than $\phi_1$, so that new clauses of $\phi_2$ can be inserted in the place of the removed clauses of $\phi_1$ when an inference rule is applied. Therefore, we do not need dynamic memory management and the implementation can be faster.

Rule \ref{resolution} for $k$=2 and Rule \ref{linear1} can be applied using a matching algorithm \cite<see, e.g.,>[for an efficient implementation]{CLRS01} over the lists of clauses. The first has a time complexity of $O(m)$, where $m$ is the number of clauses in the CNF formula. The second has a time complexity of $O(u)$, where $u$ is the number of unit clauses in the CNF formula. These rules are applied at every node, before any lower bound computation or application of other inference rules. Rule \ref{resolution} ($k$=2) is applied
as many times as possible to derive unit clauses before applying Rule \ref{linear1}.

The implementation of Rule \ref{linear2}, Rule \ref{linear3}, Rule \ref{nonlinear1}, and Rule \ref{nonlinear2} is entirely based on unit propagation. Given a CNF formula $\phi$, unit propagation constructs an implication graph $G$ \cite<see, e.g.,>{bks03}, from which the applicability of inference rules is detected. In this section, we first describe the construction of the implication graph, and then describe how to determine the applicability of Rule~\ref{linear2}, Rule~\ref{linear3}, Rule~\ref{nonlinear1}, and Rule~\ref{nonlinear2}. Then, we analyze the complexity, termination and (in)completeness of the application of 
the rules. Finally we discuss the extension of the inference rules to weighted Max-SAT and their implementation.

\subsection{Implication Graph}

Given a CNF formula $\phi$, Figure \ref{UP} shows how unit propagation constructs an implication graph whose nodes are literals.

\begin{figure} 
\begin{algorithmic}
\REQUIRE{$UnitPropagation(\phi)$ :  $\phi$ is a CNF formula not containing the complementary unit clauses $\ell$ and $\bar \ell$ for any literal $\ell$}
\STATE initialize $G$ as the empty graph
\STATE add a node labeled with $\ell$  for every literal $\ell$ in a unit clause $c$ of $\phi$
\REPEAT
\IF{$\ell_1$, $\ell_2$, ..., $\ell_{k-1}$ are nodes of $G$, $c$ = $\bar \ell_1 \vee \bar \ell_2  \vee ... \vee \bar \ell_{k-1} \vee \ell_k$ is a clause of $\phi$,  and $\ell_k$ is not a node of $G$,}
\STATE add into $G$ a node labeled with $\ell_k$
\STATE add into $G$ a directed edge from node $\ell_i$ to $\ell_{k}$ for every $i$ ($1\le i < k$)
\ENDIF
\UNTIL{no more nodes can be added or there is a literal $\ell$ such that both $\ell$ and $\bar \ell$ are nodes of $G$}
\STATE \textbf{Return} $G$
\ENSURE{Implication graph $G$ of $\phi$}
\end{algorithmic}
\caption{Unit propagation for constructing implication graphs} \label{UP}
\end{figure}

Note that every node in $G$ corresponds to a different literal, where $\ell$ and $\bar \ell$ are considered as different literals. When the CNF formula contains several copies of a unit clause $\ell$, the algorithm adds just one node with label $\ell$.

\begin{example} \label{ImplicationGraph} 
Let $\phi$=\{$x_1, x_1, \bar x_1 \vee x_2, \bar x_1 \vee x_3, \bar x_2 \vee \bar x_3 \vee x_4,
x_5, \bar x_5 \vee x_6, \bar x_5 \vee x_7, \bar x_6 \vee \bar x_7 \vee \bar x_4, \bar x_5 \vee x_8 $\}. $UnitPropagation$
constructs the implication graph of Figure~\ref{UP-example}, in which we add a special node
$\Box$ to highlight the contradiction.

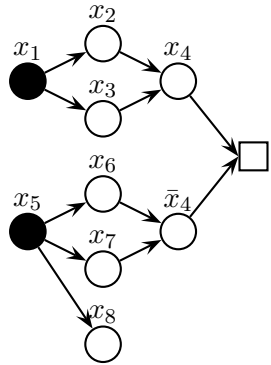
\begin{figure} 
\begin{center}
\begin{pspicture}(0,0)(2,4)
\psset{arrowscale=1.7}
\Cnode*(0,3.5){X1}
\Cnode(1,4){X2}
\Cnode(1,3){X3}
\Cnode(2,3.5){X4}
\Cnode*(0,1.5){X5}
\Cnode(1,2){X6}
\Cnode(1,1){X7}
\Cnode(2,1.5){NX4}
\Cnode(1,0){X8}
\pnode(2.8,2.5){E1}
\rput(3,2.5){\psframe(-.2,-.2)(.2,.2)}
\rput(0,3.9){$x_1$}
\rput(1,4.4){$x_2$}
\rput(1,3.4){$x_3$}
\rput(2,3.9){$x_4$}
\rput(0,1.9){$x_5$}
\rput(1,2.4){$x_6$}
\rput(1,1.4){$x_7$}
\rput(2,1.91){$\bar x_4$}
\rput(1,0.4){$x_8$}
\ncline{->}{X1}{X2}
\ncline{->}{X1}{X3}
\ncline{->}{X2}{X4}
\ncline{->}{X3}{X4}
\ncline{->}{X5}{X6}
\ncline{->}{X5}{X7}
\ncline{->}{X5}{X8}
\ncline{->}{X6}{NX4}
\ncline{->}{X7}{NX4}
\ncline{->}{X4}{E1}
\ncline{->}{NX4}{E1}
\end{pspicture}
\end{center}
\caption{Example of implication graph} \label{UP-example}
\end{figure}
\end{example}

$G$ is always acyclic because every added edge connects a new node. It is well known that the time
complexity of unit propagation is $O(|\phi|)$, where $|\phi|$ is the size of $\phi$ \cite<see, e.g.,>{Freeman95}.

We associate clause $c$=$\bar \ell_1 \vee \bar \ell_2  \vee ... \vee \bar \ell_{k-1} \vee \ell_k$ with node $\ell_k$ if node $\ell_k$ is added into $G$ because of $c$. Note that node $\ell_k$ does not have any incoming edge if and only if $c$ is unit ($k$=1), and the node has only one incoming edge if and only if $c$ is binary ($k$=2). Once $G$ is constructed, if $G$ contains both $\ell$ and $\bar \ell$ for some literal $\ell$ (i.e.,  unit propagation deduces a contradiction), it is easy to identify all nodes from which there exists a path to $\ell$ or $\bar \ell$ in $G$; i.e., the clauses implying $\ell$ or $\bar \ell$. All these clauses constitute an inconsistent subset $S$ of $\phi$. In the above example, clauses $x_1, \bar x_1 \vee x_2, \bar x_1 \vee x_3$ and $\bar x_2 \vee \bar x_3 \vee x_4$ imply
$x_4$, and clauses 
$x_5, \bar x_5 \vee x_6, \bar x_5 \vee x_7$ and $\bar x_6 \vee \bar x_7 \vee \bar x_4$ imply $\bar x_4$. Clause
$\bar x_5 \vee x_8$ does not contribute to the contradiction. The inconsistent subset $S$ is \{$x_1, \bar x_1 \vee x_2, \bar x_1 \vee x_3, \bar x_2 \vee \bar x_3 \vee x_4,
x_5, \bar x_5 \vee x_6, \bar x_5 \vee x_7, \bar x_6 \vee \bar x_7 \vee \bar x_4$\}. 

\subsection{Applicability of Rule \ref{linear2}, Rule \ref{linear3}, Rule \ref{nonlinear1}, and Rule \ref{nonlinear2}}

We assume that unit propagation deduces a contradiction and, therefore, the implication graph $G$ contains both $\ell$ and 
$\bar \ell$ for some literal $\ell$. Let $S_\ell$ be the set of all nodes from which there exists a path to $\ell$,
let $S_{\bar \ell}$ be the set of all nodes from which there exists a path to $\bar \ell$, 
and let $S$=$S_\ell \cup S_{\bar \ell}$. As a clause is associated with each node in $G$, we also
use $S$, $S_\ell$, and $S_{\bar \ell}$ to denote the set of clauses associated with the nodes in 
$S$, $S_\ell$, and $S_{\bar \ell}$,
respectively. Lemma~\ref{lemma2} and Lemma~\ref{lemma3} are used to detect the applicability of Rule~\ref{linear2}, Rule~\ref{linear3}, Rule~\ref{nonlinear1}, and Rule~\ref{nonlinear2}.

\begin{lemma} \label{lemma2} 
Rule~\ref{linear2} and Rule~\ref{linear3} are applicable if
\begin{enumerate}
\item in $S_\ell$ (resp. $S_{\bar \ell}$), there is one unit clause and all the other clauses are binary,
\item nodes in $S_\ell$ (resp. $S_{\bar \ell}$) form an implication chain starting at the unit clause and ending at $\ell$ (resp. $\bar \ell$),
\item $S_\ell \cap S_{\bar \ell}$ is empty.
\end{enumerate}
\end{lemma}

  \begin{proof}
Starting from the node corresponding to the unit clause in $S_\ell$ (resp. $S_{\bar \ell}$), and following in parallel the two implication chains, 
we have $\phi_1$ in Rule~\ref{linear2} or Rule~\ref{linear3} by writing down the clause corresponding to each node. \Cqfd
  \end{proof}

\begin{example}\label{Rule3and4}
Let $\phi$ be the following CNF formula containing clauses $c_1$ to $c_7$:
$\{c_1: x_1, ~~c_2: \bar x_1 \vee x_2, ~~c_3: \bar x_2 \vee x_3,  ~~c_4: \bar x_3 \vee  x_4, ~~c_5: x_5, ~~c_6: \bar x_5 \vee  x_6, ~~c_7: \bar x_6 \vee  \bar x_4\}$.
Unit propagation constructs the implication graph shown in Figure~\ref{imp-graph1}, which contains the complementary literals $x_4$ and $\bar x_4$. 
 
\begin{figure} 
\begin{center}
\begin{pspicture}(0,0)(4,2)
\psset{arrowscale=1.7}
\cnodeput(0,1){X1}{$c_1$}
\cnodeput(1,1){X2}{$c_2$}
\cnodeput(2,1){X3}{$c_3$}
\cnodeput(3,1){X4}{$c_4$}
\cnodeput(0,0){X5}{$c_5$}
\cnodeput(1,0){X6}{$c_6$}
\cnodeput(2,0){NX4}{$c_7$}
\pnode(3.8,.5){E1}
\rput(4,.5){\psframe(-.2,-.2)(.2,.2)}
\rput(0,1.47){$x_1$}
\rput(1,1.47){$x_2$}
\rput(2,1.47){$x_3$}
\rput(3,1.47){$x_4$}
\rput(0,0.47){$x_5$}
\rput(1,0.47){$x_6$}
\rput(2,0.48){$\bar x_4$}
\ncline{->}{X1}{X2}
\ncline{->}{X2}{X3}
\ncline{->}{X3}{X4}
\ncline{->}{X5}{X6}
\ncline{->}{X6}{NX4}
\ncline{->}{X4}{E1}
\ncline{->}{NX4}{E1}
\end{pspicture}
\end{center}
\caption{Example of implication graph} \label{imp-graph1}
\end{figure}
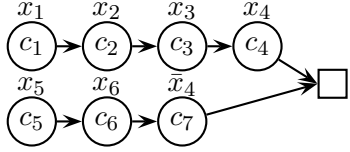
 
 Rule  \ref{linear3} is applicable because $\ell$=$x_4$, $S_\ell$=\{$x_1 (c_1), ~~x_2 (c_2), ~~x_3 (c_3),  ~~x_4 (c_4)$\}, and $S_{\bar \ell}$=\{$x_5 (c_5), ~~x_6 (c_6), ~~\bar x_4 (c_7)$\}. It is easy to verify that the three conditions of 
Lemma~\ref{lemma2} are satisfied.
 
 \end{example}
    
 \noindent {\bf Remark}: $\phi$ can be rewritten as $\{c_1: x_1, ~~c_2: \bar x_1 \vee x_2, ~~c_3: \bar x_2 \vee x_3,  ~~c_4: \bar x_3 \vee  x_4, ~~c_7: \bar x_4 \vee \bar x_6, ~~c_6:  x_6 \vee \bar x_5 , ~~c_5: x_5\}$ to be compared with $\phi_1$ in Rule 
\ref{linear3}.

The application of Rule~\ref{linear2} and Rule~\ref{linear3} consists of replacing every binary
clause $c$ in $S$ with a binary clause obtained by negating every literal of $c$, removing the two unit clauses
of $S$ from $\phi$, and incrementing  \#$emptyClauses$($\phi$) by 1.

\begin{lemma} \label{lemma3} 
Rule~\ref{nonlinear1} and Rule~\ref{nonlinear2} are applicable if 

\begin{enumerate}
\item in $S$=$S_\ell \cup S_{\bar \ell}$, there is one unit clause and all the other clauses are binary;
 i.e., all nodes in $S$, except for the node corresponding to the unit clause, have exactly one incoming edge in $G$.
\item $S_\ell \cap S_{\bar \ell}$ is non-empty and contains $k$ ($k >$0) nodes forming an implication chain of the form $\ell_1  \rightarrow \ell_2  \rightarrow~ \cdots~  \rightarrow \ell_{k}$, where $\ell_{k}$ is the last node of the chain.
\item ($S_\ell \cup S_{\bar \ell}$)-($S_\ell \cap S_{\bar \ell}$) contains exactly three nodes : $\ell$, $\bar \ell$, and a third one. Let $\ell_{k+1}$ be the third literal,

if $\ell_{k+1}$ $\in$ $S_\ell$, then $G$ contains the following implications
$$\ell_k \rightarrow \ell_{k+1} \rightarrow \ell$$
$$\ell_k \rightarrow \bar \ell$$

if $\ell_{k+1}$ $\in$ $S_{\bar \ell}$, then $G$ contains the following implications
$$\ell_k \rightarrow \ell$$
$$\ell_k \rightarrow \ell_{k+1} \rightarrow \bar \ell$$
\end{enumerate}
\end{lemma}

\begin{proof}
Assume, without loss of generality, that $\ell_{k+1}$ $\in$ $S_\ell$; the case $\ell_{k+1}$ $\in$ $S_{\bar \ell}$ 
is symmetric. The implication chain formed by the nodes of $S_\ell \cap S_{\bar \ell}$ corresponds to 
the clauses \{$\ell_1$, $\bar \ell_1 \vee \ell_2$, \ldots, $\bar \ell_{k-1} \vee \ell_k$\}, which, together with the 
three clauses \{$\bar \ell_k \vee \ell_{k+1}, \bar \ell_{k+1} \vee \ell$, $\bar \ell_k \vee \bar \ell$\}  
corresponding to $\ell_k \rightarrow \ell_{k+1} \rightarrow \ell$ and $\ell_k \rightarrow \bar \ell$, give
 $\phi_1$ in Rule \ref{nonlinear1} or Rule \ref{nonlinear2}.  \Cqfd
 \end{proof}

\begin{example}\label{Rule5and6}
Let $\phi$ be the following CNF formula containing clauses $c_1$ to $c_5$:
\{$c_1: x_1, ~~c_2: \bar x_1 \vee x_2, 
~~c_3: \bar x_2 \vee x_3, ~~c_4: \bar x_2 \vee  x_4, ~~c_5: \bar x_3 \vee \bar x_4$\}. 
Unit propagation constructs the implication graph shown in Figure~\ref{imp-graph2}, which contains the complementary literals $x_4$ and $\bar x_4$. 

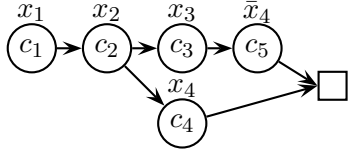
\begin{figure} 
\begin{center}
\begin{pspicture}(0,0)(4,2)
\psset{arrowscale=1.7}
\cnodeput(0,1){X1}{$c_1$}
\cnodeput(1,1){X2}{$c_2$}
\cnodeput(2,1){X3}{$c_3$}
\cnodeput(2,0){NX4}{$c_4$}
\cnodeput(3,1){X4}{$c_5$}
\pnode(3.8,.5){E1}
\rput(4,.5){\psframe(-.2,-.2)(.2,.2)}
\rput(0,1.47){$x_1$}
\rput(1,1.47){$x_2$}
\rput(2,1.47){$x_3$}
\rput(3,1.48){$\bar x_4$}
\rput(2,0.47){$x_4$}
\ncline{->}{X1}{X2}
\ncline{->}{X2}{X3}
\ncline{->}{X3}{X4}
\ncline{->}{X2}{NX4}
\ncline{->}{X4}{E1}
\ncline{->}{NX4}{E1}
\end{pspicture}
\end{center}
\caption{Example of implication graph} \label{imp-graph2}
\end{figure}

We have $S_{x_4}$=\{$x_1(c_1)$,  $ x_2 (c_2)$, $x_4 (c_4)$\} and $S_{\bar x_4}$=\{$x_1(c_1)$,  
$x_2 (c_2)$, $x_3 (c_3)$, $\bar x_4 (c_5)$\}. The nodes in $S_{x_4} \cap S_{\bar x_4}$
obviously form an implication chain: $x_1 \rightarrow x_2$.  ($S_{x_4} \cup S_{\bar x_4}$)-($S_{x_4} \cap S_{\bar x_4}$)=\{$x_3 (c_3)$, $x_4 (c_4)$, $\bar x_4 (c_5)$\}. $G$ contains $x_2  \rightarrow x_3  \rightarrow \bar x_4$ and $x_2  \rightarrow x_4$. Rule \ref{nonlinear2} is applicable.
 \end{example}
 
The  application of Rule \ref{nonlinear1} and Rule \ref{nonlinear2} consists of removing the unit clause of 
$S_\ell \cup S_{\bar \ell}$ 
from $\phi$, replacing
each binary clause $c$ in $S_\ell \cap S_{\bar \ell}$ with a binary clause obtained from $c$ by negating the two literals of $c$, 
replacing the three binary clauses in ($S_\ell \cup S_{\bar \ell}$)-($S_\ell \cap S_{\bar \ell}$) with two ternary
clauses, and incrementing \#$emptyClauses$($\phi$) by 1.

\subsection{Complexity, Termination, and (In)Completeness of Rule Applications}

In our branch and bound algorithm for Max-SAT, 
we combine the application of the inference rules and the
computation of the underestimation of the lower bound. 
Given a CNF formula $\phi$, function $underestimation$ uses unit propagation to construct an implication graph $G$. 
Once $G$ contains two nodes $\ell$ and $\bar \ell$ for some literal $\ell$, $G$ is analyzed to determine whether some
inference rule is applicable. If some rule is applicable, it is applied and $\phi$ is transformed into an equivalent
Max-SAT instance. Otherwise, all clauses contributing to the contradiction are removed from $\phi$, and the underestimation is incremented by 1. This procedure is repeated until unit propagation cannot derive more contradictions. Finally, all removed clauses, except those removed or replaced due to inference rule applications, are reinserted into $\phi$. The underestimation, together with the new $\phi$, is returned.

It is well known that unit propagation can be implemented with a time complexity linear in the size of 
$\phi$ \cite<see, e.g.,>{Freeman95}. 
The complexity of determining the applicability of the inference rules using Lemma \ref{lemma2} 
and Lemma \ref{lemma3} is linear in the size of $G$,  bounded by the number of literals in $\phi$, if we assume
that the graph is represented by doubly-linked lists. 
The application of an inference rule is obviously linear in the size of $G$. So, the whole time complexity of function $underestimation$ with inference rule applications is $O(d*|\phi|)$, where $d$ is the number of contradictions that function $underestimation$ is able to detect using unit propagation.
Observe that the factor $d$ is needed because the application of the rules inserts new clauses in the place of the removed clauses.

Since every inference rule application reduces the size of $\phi$,  function $underestimation$ with 
inference rule applications has linear space complexity, and it always terminates. 
Recall that new clauses added by the inference rules can be stored
in the place of the old ones. The data structures for loading $\phi$ can be statically and efficiently maintained.

We have proved that the inference rules are sound.
The following example shows that the application of the rules is not necessarily complete in our implementation,
in the sense that not all possible applications of the inference rules are necessarily done.

\begin{example} \label{incompleteness}
Let $\phi$=\{$x_1, x_3, x_4, \bar x_1 \vee \bar x_3 \vee \bar x_4, \bar x_1 \vee \bar x_2, x_2$\}.
Unit propagation may discover the inconsistent subset $S$=\{$x_1, x_3, x_4, \bar x_1 \vee \bar x_3 \vee \bar x_4$\}. In this case, no inference rule is applicable to $S$. Then, the underestimation of the lower bound is incremented by~1, and $\phi$ becomes \{$\bar x_1 \vee \bar x_2, x_2$\}. Unit propagation cannot detect more contradictions in $\phi$, 
and function $underestimation$ stops after reinserting \{$x_1, x_3, x_4, \bar x_1 \vee \bar x_3 \vee \bar x_4$\} into $\phi$. The value 1 is returned, together with the unchanged $\phi$. Note that Rule~\ref{linear2} is applicable to the subset \{$x_1, \bar x_1 \vee \bar x_2, x_2$\}  of $\phi$, but is not applied.

Actually,  function $underestimation$ only applies Rule~\ref{linear2} if unit propagation detects the inconsistent subset \{$x_1, \bar x_1 \vee \bar x_2, x_2$\} instead of \{$x_1, x_3, x_4, \bar x_1 \vee \bar x_3 \vee \bar x_4$\}.
The detection of an inconsistent subset depends on the ordering in which unit clauses are propagated in unit propagation.
In this example, the inconsistent subset \{$x_1, \bar x_1 \vee \bar x_2, x_2$\} is discovered if unit clause $x_2$ is propagated before $x_3$ and $x_4$. Further study is needed to define orderings for unit clauses that maximize the application of inference rules.
\end{example}

Observe that our algorithm is deterministic, and always computes the same lower bound if the order of clauses is not changed.

\subsection{Inference Rules for Weighted Max-SAT}
\label{weightedRule}

The inference rules presented in this paper can be naturally extended to weighted Max-SAT. In weighted Max-SAT, every clause is associated with a weight and the problem consists of finding a truth assignment for which the sum of the weights of unsatisfied clauses is minimum. For example, 
the weighted version of Rule \ref{linear2} could be

\begin{theoreme}  \label{weightedLinear2} 
If ${\phi}_1$=\{$(l_1, w_1), ~(\bar{l}_1 \vee \bar{l}_2, w_2), ~(l_2, w_3)\} \cup \phi'$, then ${\phi}_2$=\{$(\Box, w),~ (l_1 \vee l_2, w),~(l_1, w_1-w), ~(\bar{l}_1 \vee \bar{l}_2, w_2-w), ~(l_2, w_3-w)\} \cup \phi'$ is equivalent to ${\phi}_1$
\end{theoreme}

\noindent where $w_1$, $w_2$ and $w_3$ are positive integers representing the clause weight, and $w$=min($w_1$, $w_2$, $w_3$). Mandatory clauses, that have to be satisfied in any optimal solution, are specified with the weight $\infty$. Note that if $w$$\neq$$\infty$, $\infty$-$w$=$\infty$ and if $w$=$\infty$, no optimal solution can be found and the solver should backtrack. Clauses with weight 0 are removed. Observe that $\phi_1$ can be rewritten as $\phi_{11} \cup \phi_{12}$,
where $\phi_{11}$=\{$(l_1, w), ~(\bar{l}_1 \vee \bar{l}_2, w), ~(l_2, w)$\}, and $\phi_{12}$=\{$(l_1, w_1-w), ~(\bar{l}_1 \vee \bar{l}_2, w_2-w), ~(l_2, w_3-w)$\} $\cup$ $\phi'$. Then, the weighted inference rule is equivalent to the unweighted version applied $w$ times to the (unweighted) clauses of $\phi_{11}$.

Similarly, the weighted version of Rule \ref{linear3} could be

\begin{theoreme} \label{weightedlinear3}
If ${\phi}_1$=\{$(l_1, w_1) ~({\bar l}_1 \vee l_2, w_2), ~({\bar l}_2 \vee l_3, w_3),~ \ldots,~({\bar l}_{k} \vee l_{k+1}, w_{k+1}), ~({\bar l}_{k+1}, w_{k+2})\} \cup \phi'$,
then ${\phi}_2$=\{$(\Box,w), ~(l_1 \vee {\bar l}_2, w), ~(l_2 \vee {\bar l}_3, w), ~ \ldots, ~(l_{k} \vee {\bar l}_{k+1}, w), ~(l_1, w_1-w), ~({\bar l}_1 \vee l_2, w_2-w), ~({\bar l}_2 \vee l_3, w_3-w),~ \ldots,~({\bar l}_{k} \vee l_{k+1}, w_{k+1}-w), ~({\bar l}_{k+1}, w_{k+2}-w)\} \cup \phi'$ is equivalent to
${\phi}_1$
\end{theoreme}

\noindent where $w$=min($w_1$, $w_2$, \ldots, $w_{k+2}$). Observe that $\phi_1$ can also be rewritten as $\phi_{11} \cup \phi_{12}$, with $\phi_{11}$=\{$(l_1, w) ~({\bar l}_1 \vee l_2, w), ~({\bar l}_2 \vee l_3, w),~ \ldots,~({\bar l}_{k} \vee l_{k+1}, w), ~({\bar l}_{k+1}, w)$\}, The weighted version of Rule \ref{linear3} is equivalent to the unweighted Rule \ref{linear3} applied $w$ times to the (unweighted) clauses of $\phi_{11}$.

The current implementation of the inference rules can be naturally extended to weighted inference rules. If an inconsistent subformula is detected and a rule is applicable (clause weights are not considered in the detection of the inconsistent subformula and of the applicability of the rule, provided that clauses with weight 0 have been discarded),  then $\phi_{11}$ and $\phi_{12}$ are separated after computing the minimal weight $w$ of all clauses in the detected inconsistent subformula, and the rule is applied to $\phi_{11}$. The derived clauses and clauses in $\phi_{12}$ can be used in subsequent reasoning.

\section{MaxSatz: a New Max-SAT Solver} \label{MaxSatz}

We have implemented a new Max-SAT solver, called MaxSatz, that incorporates the lower bound computation method based on 
unit propagation defined in Section~\ref{basic-solver},
and applies  the inference rules defined in Section~\ref{inference-rules}. The name of MaxSatz comes from the fact that the implementation of 
our algorithm incorporates most of the technology that was developed for the SAT solver Satz~\cite{LA97a,LA97b}.

MaxSatz incorporates the lower bound based on unit propagation, and applies
Rule~\ref{resolution},  Rule~\ref{linear1}, Rule~\ref{linear2}, Rule~\ref{linear3}, Rule~\ref{nonlinear1}, and Rule~\ref{nonlinear2}. In addition, MaxSatz applies the following techniques:

\begin{itemize}

\item Pure literal rule: If a literal only appears with either positive polarity or negative polarity, we delete the clauses containing that literal.

\item Empty-Unit clause rule \cite{AMP03a}:  Let $neg1(x)$ ($pos1(x)$) be the number of unit clauses in which $x$ is negative (positive). If $\#emptyClauses(\phi)+neg1(x) \geq UB$, then we assign $x$ to false. If $\#emptyClauses(\phi)+pos1(x) \geq UB$, then we assign $x$ to true.

\item Dominating Unit Clause (DUC) rule \cite{NR00}: If  the  number of clauses  in which a literal $x$ ($\bar x$) appears is 
not greater than $neg1(x)$ ($pos1(x)$), then we assign~$x$ to false (true). 

\item Variable selection heuristic: Let $neg2(x)$ ($pos2(x)$) be the number of binary clauses in which $x$ is negative 
(positive), and let $neg3(x)$ ($pos3(x)$) be the number of clauses containing three or more literals in which $x$ is negative (positive). 
We select the variable $x$ such that ($neg1(x)+4*neg2(x)+neg3(x)$)*($pos1(x)+4*pos2(x)+pos3(x)$) is the largest. The fact that binary clauses are counted four times more than other clauses was determined empirically.

\item Value selection heuristic: Let $x$ be the selected branching variable. If $neg1(x)+4*neg2(x)+neg3(x)<pos1(x)+4*pos2(x)+pos3(x)$, set $x$ to true. Otherwise set $x$ to false. This heuristics was also determined empirically.
\end{itemize}

In this paper, in order to compare the inference rules defined, we have used three simplified versions of MaxSatz:
\begin{itemize}
\item MaxSat0: does not apply any inference rule defined in Section~\ref{inference-rules}.

\item MaxSat12: applies rules 1 and 2, but not rules 3, 4, 5 and 6.

\item MaxSat1234: applies rules 1, 2, 3 and 4, but not rules 5 and 6.

\end{itemize}

Actually, MaxSatz corresponds to MaxSat123456 in our terminology. MaxSat12 corresponds to an improved version of 
the solver $UP$ \cite{LMP05}, using a special ordering for propagating unit clauses in unit propagation. 
MaxSat12 maintains two queues during 
unit propagation: $Q_1$ and $Q_2$. When  MaxSat12 starts the search for an inconsistent subformula via unit propagation,
$Q_1$ contains all the unit clauses of the CNF formula under consideration (more recently
derived unit clauses are at the end of $Q_1$), and $Q_2$ is empty. The unit clauses derived during the application
of unit propagation are stored in $Q_2$, and unit propagation
does not use any unit clause from $Q_1$ unless $Q_2$ is empty. Intuitively, this ordering prefers unit clauses which were non-unit clauses before starting the application of unit propagation. This way, the derived inconsistent subset contains, in
general, less unit clauses. The unit clauses which have not been consumed will contribute to detect further inconsistent subsets. Our experimental results~\cite{LMP06} show that the search tree size of MaxSat12 is substantially smaller than that of UP, and MaxSat12 is substantially faster than UP. MaxSat0, Maxsat1234, and MaxSatz use the same ordering as MaxSat12 for propagating unit clauses in unit propagation.

The source code of MaxSat0, MaxSat12, MaxSat1234, and MaxSatz can be found at 
http://web.udl.es/usuaris/m4372594/jair-maxsatz-solvers.zip, and at
http://www.laria.u-picardie.fr/\~{}cli/maxsatz.tar.gz.

\section{Experimental Results}
\label{Experiments}

We report on the experimental investigation performed for unweighted Max-SAT in order to evaluate the inference rules defined in
Section~\ref{inference-rules}, and to compare MaxSatz with the best performing state-of-the-art solvers that were publicly available when this paper was submitted. 
The experiments were performed on a Linux Cluster with processors 2~GHz AMD Opteron with 1~Gb of RAM.

The structure of this section is as follows. We first describe the solvers and benchmarks that we have considered. Then,
we present the experimental evaluation of the inference rules. Finally, we show the experimental comparison
of MaxSatz with other solvers.

\subsection{Solvers and Benchmarks} \label{solvers-benchmarks}

MaxSatz was compared with the following Max-SAT solvers:

\begin{itemize}

\item \texttt{BF}\footnote{Downloaded in October 2004 from http://infohost.nmt.edu/\~{}borchers/satcodes.tar.gz}~\cite{BF99}: a branch and bound Max-SAT solver which
uses MOMS as dynamic variable selection heuristic and does not consider underestimations in the
computation of the lower bound.
It was developed by Borchers and Furman in 1999.

\item \texttt{AGN}\footnote{Downloaded in October 2005 from 
http://www-fs.informatik.uni-tuebingen.de/\~{}gramm/ }~\cite{AGN98}: a branch and bound Max-2SAT solver.
It was developed by Alber, Gramm and Niedermeier in 1998.

\item \texttt{AMP}\footnote{Available at http://web.udl.es/usuaris/m4372594/software.html}~\cite{AMP03}: a branch and bound Max-SAT solver based on BF that incorporates a lower bound of better quality and
the Jeroslow-Wang variable selection heuristic~\cite{JeroslowWang90}. It was developed by Alsinet, Many\`a and Planes and  presented
at SAT-2003.

\item \texttt{toolbar}\footnote{Downloaded in October 2005 from http://carlit.toulouse.inra.fr/cgi-bin/awki.cgi/ToolBarIntro}~\cite{GLMS03,LH05a}: a Max-SAT solver whose inference was inspired in soft arc consistency properties implemented in weighted CSP solvers.
It was developed by  de~Givry, Larrosa, Meseguer and Schiex and was first presented at CP-2003.
We used version 2.2 with default parameters.

\item \texttt{MaxSolver\footnote{Downloaded in October 2005 from http://cic.cs.wustl.edu/maxsolver/ }}~\cite{XZ04}: a branch and bound Max-SAT solver that applies a number of efficient
inference rules. It was developed by Xing and Zhang
and presented at CP-2004. We used the second release of this solver.

\item \texttt{Lazy}\footnote{Available at http://web.udl.es/usuaris/m4372594/software.html}~\cite{AMP05}: a branch and bound Max-SAT 
solver with lazy data structures and a static
variable selection heuristic. It was developed by Alsinet, Many\`a and Planes and  presented
at SAT-2005.

\item \texttt{UP\footnote{Available at http://web.udl.es/usuaris/m4372594/software.html}}~\cite{LMP05}: a branch and bound Max-SAT solver with the lower bound computation method
based on unit propagation (cf.~Section~\ref{basic-solver}). It was developed by Li, Many\`a and Planes and  presented
at CP-2005.

\end{itemize}


We used as benchmarks randomly generated Max-2SAT instances and Max-3SAT instances, graph 3-coloring 
instances\footnote{Given an undirected graph $G=(V,E)$, where $V=\{x_1, \ldots, x_n\}$
is the set of vertices and $E$ is the set of edges, and a set of three colors, the graph 3-coloring problem is 
the problem of coloring every vertex with one of the three colors in such a way that, for each edge $(x_i, x_j) \in E$, vertex $x_i$ and vertex $x_j$ do not have the same color.}, as well as
Max-Cut instances\footnote{Given an undirected graph $G=(V,E)$, let $w_{x_i,x_j}$ be the weight associated with
each edge $(x_i, x_j) \in E$. The weighted Max-Cut problem is to find a subset $S$ of $V$ such that
$W(S, \overline{S}) = \sum_{x_i \in S, x_j \in \overline{S}} w_{x_i,x_j}$ is maximized, where
$\overline{S}= V- S$. In this paper, we set weight $w_{x_i,x_j}=1$ for all edges.}.
We also considered the unweighted Max-SAT benchmarks submitted to the Max-SAT Evaluation 2006, including Max-Cut, Max-Ones, Ramsey numbers, and random Max-2SAT and Max-3SAT instances. 

We generated Max-2SAT instances and  Max-3SAT instances using the generator \texttt{mwff.c}  
developed by Bart Selman, which allows for duplicated clauses. For Max-Cut, we first generated
a random graph of $m$ edges in which every edge is randomly selected among the set of all possible edges. If the graph is not connected,  it is discarded. If the graph is connected, we used the 
encoding of~\citeauthor{SZ05}~\citeyear{SZ05} to encode the Max-Cut instance into a CNF:  we created, for each 
edge $(x_i, x_j)$, exactly
two binary clauses $(x_i \vee x_j)$ and  $(\bar x_i \vee \bar x_j)$. If $\phi$ is the collection of
such binary clauses, then the Max-Cut instance has a cut of weight $k$ iff the Max-SAT instance has
an assignment under which $m+k$ clauses are satisfied. 

For graph 3-coloring, we first used 
Culberson's generator to generate a random $k$-colorable graph of type IID (independent random 
edge assignment, variability=0) with $k$ vertices and a fixed edge
 density. 
We then used Culberson's converter to SAT with standard conversion and three colors 
to generate a Max-SAT instance: 
for each vertex $x_i$ and for each color $j \in \{1,2,3\}$, a propositional variable $x_{ij}$ is defined 
meaning that vertex $i$ is colored with color $j$. For each vertex $x_i$, four clauses are added to encode that the vertex is colored with exactly one color: $x_{i1} \vee x_{i2} \vee x_{i3}$, $\bar x_{i1} \vee \bar x_{i2}$,  $\bar x_{i1} \vee \bar x_{i3}$, and $\bar x_{i2} \vee \bar x_{i3}$; and, for each edge $(x_i, x_j)$, three clauses are added to encode that vertex $x_i$ and vertex $x_j$ do not have the same color: $\bar x_{i1} \vee \bar x_{j1}$, $\bar x_{i2} \vee \bar x_{j2}$, and $\bar x_{i3} \vee \bar x_{j3}$.

In random Max-2SAT and Max-3SAT instances, clauses are entirely independent to each other and do not have  structure. 
In the graph 3-coloring instances and Max-Cut instances used in this paper, clauses are not independent and have structure.
For example, in a Max-Cut instance, every time we have a clause $x_i \vee x_j$, we also have the clause $\bar x_i \vee \bar x_j$; the satisfaction of these two clauses means that the corresponding edge is in the cut. In a graph 3-coloring instance, every time we have a ternary clause $x_{i1} \vee x_{i2} \vee x_{i3}$ encoding that vertex $i$ is colored with at least a color, we also have
three binary clauses $\bar x_{i1} \vee \bar x_{i2}$, $\bar x_{i1} \vee \bar x_{i3}$, and $\bar x_{i2} \vee \bar x_{i3}$ encoding 
that vertex $i$ cannot be colored with two or more colors. Max-Cut instances only contain binary clauses, but
graph 3-coloring instances contain a ternary clause for every vertex in the graph. While 
we can derive an optimal cut from an optimal assignment of a Max-SAT encoding of any Max-Cut instance, an optimal assignment of a Max-SAT encoding of a 3-coloring instance
may assign more than one color to some vertices. 

The Max-Cut and Ramsey numbers instances from the Max-SAT Evaluation 2006 contain different structures. 
For example, the underlying graphs in the Max-Cut instances have different origins such as fault diagnosis problems, coding theory problems, and graph clique problems. Max-2SAT and Max-3SAT instances from the evaluation do not contain duplicated clauses.

We computed an initial upper bound with a local search solver for each instance.
We did not provide any parameter to any solver except the instance to be solved and the initial upper bound. 
In other words, we used the default values for all the parameters. The instances from
the Max-SAT Evaluation 2006 were solved in the same conditions as in the evaluation; i.e., no 
initial upper bound was provided to the solvers, and the maximum time allowed to solve an instance was 30 minutes.

\subsection{Evaluation of the Inference Rules}

In the first experiment performed to evaluate the impact of the inference rules of Section~\ref{inference-rules},
we solved sets of 100 random Max-2SAT instances with 50 and 100 variables; the number of clauses 
ranged from 400 to 4500 for 50 variables,
and from 400 to 1000 for 100 variables. The results obtained are shown in Figure~\ref{rule2sat}. 
Along the horizontal axis is the number of clauses, and along the vertical axis
is the mean time (left plot), in seconds, needed to solve an instance of a set,
and the mean number of branches of the proof tree (right plot). 
Notice that we use a log scale to represent both run-time and branches.

We observe that the rules are very powerful for Max-2SAT and the gain increases as the number of variables and 
the number of clauses increase.
For 50 variables and 1000 clauses (the clause to variable ratio is 20), MaxSatz is 7.6 times faster than MaxSat1234; and for 100 variables and 1000 
clauses (the clause to variable ratio is 10),  MaxSatz is 9.2 times faster than MaxSat1234. The search tree of MaxSatz is also substantially smaller than that of MaxSat1234. 
Rule~\ref{nonlinear1} and Rule~\ref{nonlinear2} are more powerful than Rule~\ref{linear2} and Rule~\ref{linear3} for Max-2SAT. 
The intuitive explanation is that MaxSatz and MaxSat1234
detect many more inconsistent subsets of clauses containing one unit clause than subsets 
containing two unit clauses, so that Rule~\ref{nonlinear1} and Rule~\ref{nonlinear2} can 
be applied many more times than Rule~\ref{linear2} and Rule~\ref{linear3} in MaxSatz.

Recall that, on the one hand, every application of Rule~\ref{linear2} and Rule~\ref{linear3} consumes two unit clauses but only produces one empty clause, limiting unit propagation in detecting more conflicts in subsequent search. On the other hand, Rule~\ref{linear2} and Rule~\ref{linear3} add clauses which may contribute to detect further conflicts. Depending on the number of clauses (or more precisely, the clause to variable ratio) in a formula, these two factors have different importance. When there are relatively few clauses, unit propagation relatively does not easily derive a contradiction from a unit clause, and the binary clauses added by Rule~\ref{linear2} and Rule~\ref{linear3} are relatively important for deriving additional conflicts and improving the lower bound, which explains why the search tree of MaxSat1234 is smaller than the search tree of MaxSat12 for instances with 100 variables and less than 600 clauses. On the contrary, 
when there are many clauses, unit propagation easily derives a contradiction from a unit clause, so that the two unit clauses consumed by Rule~\ref{linear2} and Rule~\ref{linear3}  would probably allow to derive two disjoint inconsistent subsets of clauses. In addition, the binary clauses added by Rule \ref{linear2} and Rule \ref{linear3} are relatively less important for deriving additional conflicts, considering 
the large number of clauses in the formula. In this case, the search tree of MaxSat1234 is larger than the
search tree of MaxSat12. However, in both cases, MaxSat1234 is faster that MaxSat12, meaning that the incremental 
lower bound computation due to Rule~\ref{linear2} and Rule~\ref{linear3} is very effective, since the redetection of many conflicts is avoided thanks to Rule~\ref{linear2} and 
Rule~\ref{linear3}.

\begin{figure}[H]
\begin{tabular}{cc}
\includegraphics[scale=.5]{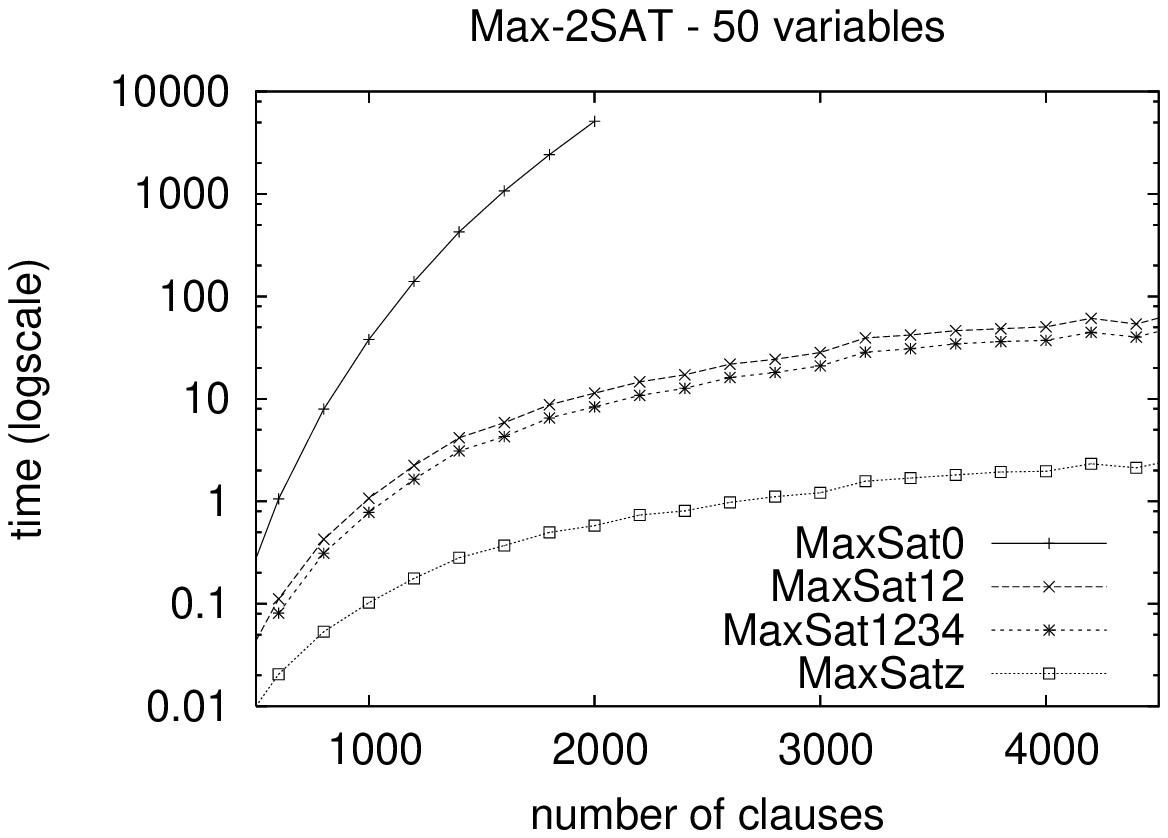} & \includegraphics[scale=.5]{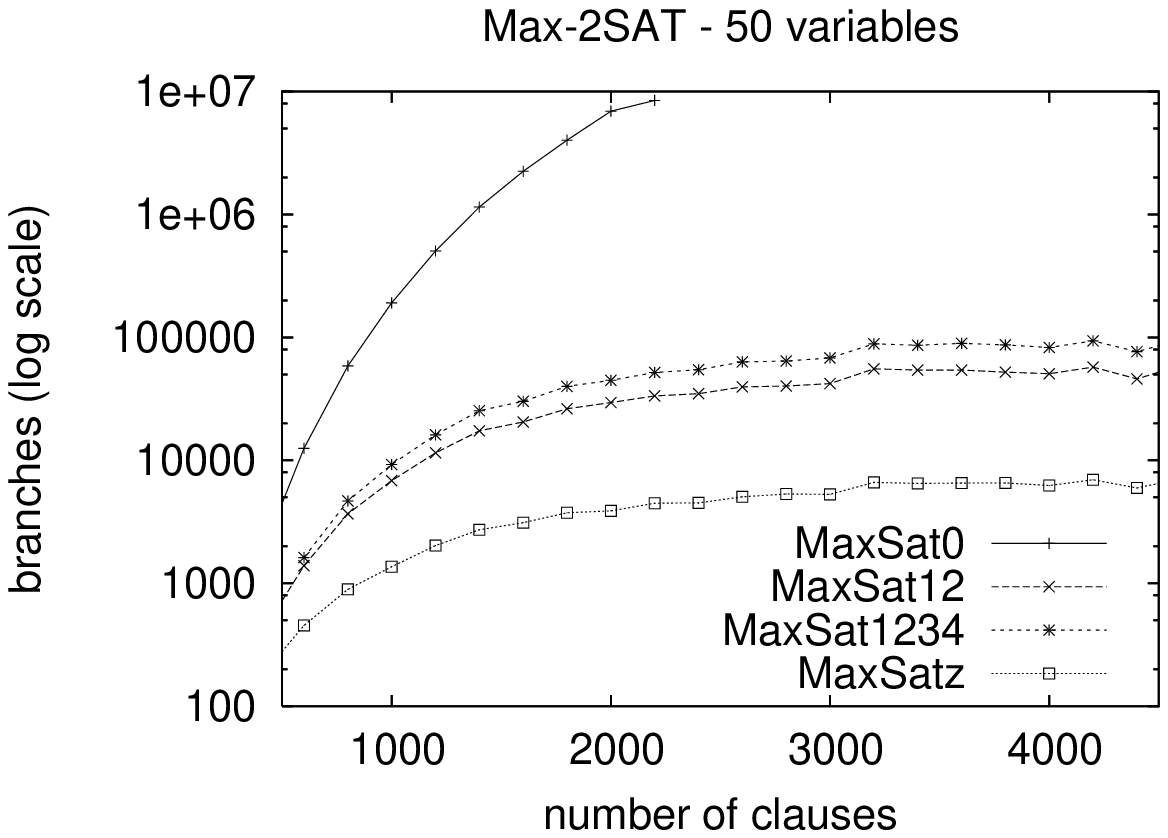} \\
\includegraphics[scale=.5]{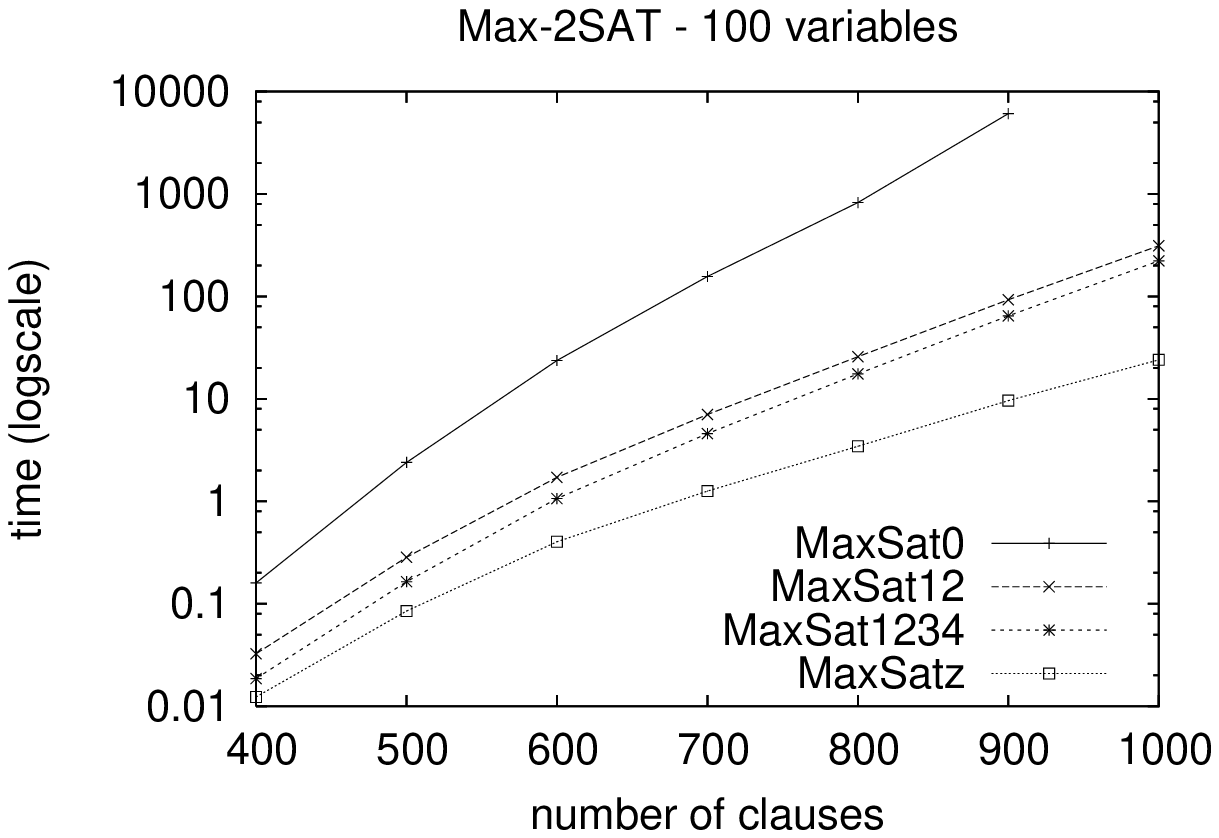} & \includegraphics[scale=.5]{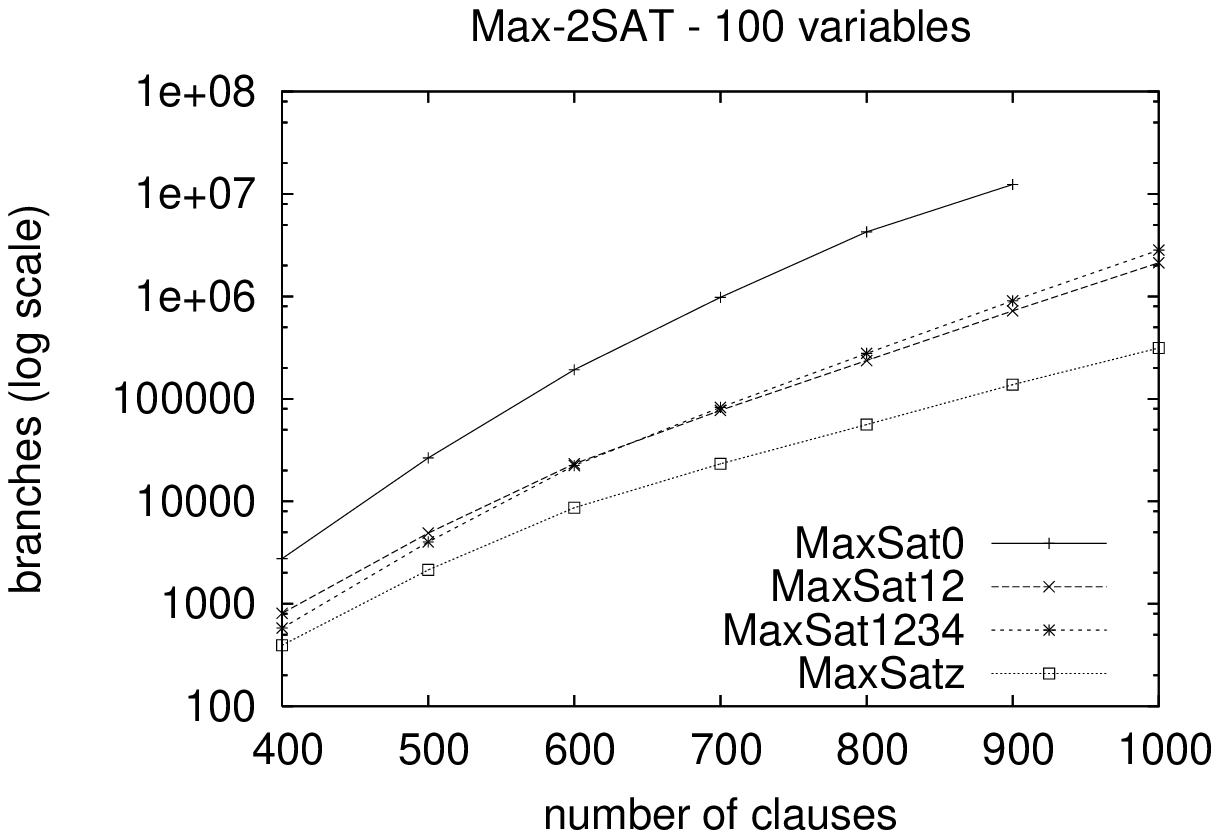} \\
\end{tabular}
\caption{\small Comparison among MaxSat12, MaxSat1234 and MaxSatz on random Max-2SAT instances.}\label{rule2sat}
\end{figure}

Rule \ref{nonlinear1} and Rule \ref{nonlinear2} do not limit unit propagation in detecting more conflicts, since their application produces one empty clause and consumes just one unit clause, which allows to derive at most one conflict in any case. 
The added ternary clauses allow to improve the lower bound, so 
that the search tree of MaxSatz is substantially smaller than the search tree of MaxSat1234. The incremental 
lower bound computation due to Rule \ref{nonlinear1} and Rule \ref{nonlinear2} also contributes to the time performance of MaxSatz.
 For example, while the search tree of MaxSatz for instances with 50 variables and 2000 clauses is about 11.5 times smaller than the search tree of MaxSat1234, 
MaxSatz is 14 times faster than MaxSat1234. 

In the second experiment, we solved random Max-3SAT instances instead of random Max-2SAT instances.
We solved instances with 50 and 70 variables; the number of clauses ranged from 400 to 1200 for 50 variables,
and from 500 to 1000 for 70 variables. The results obtained are shown in Figure~\ref{rule3sat}.

\begin{figure}[H]
\begin{tabular}{cc}
\includegraphics[scale=.5]{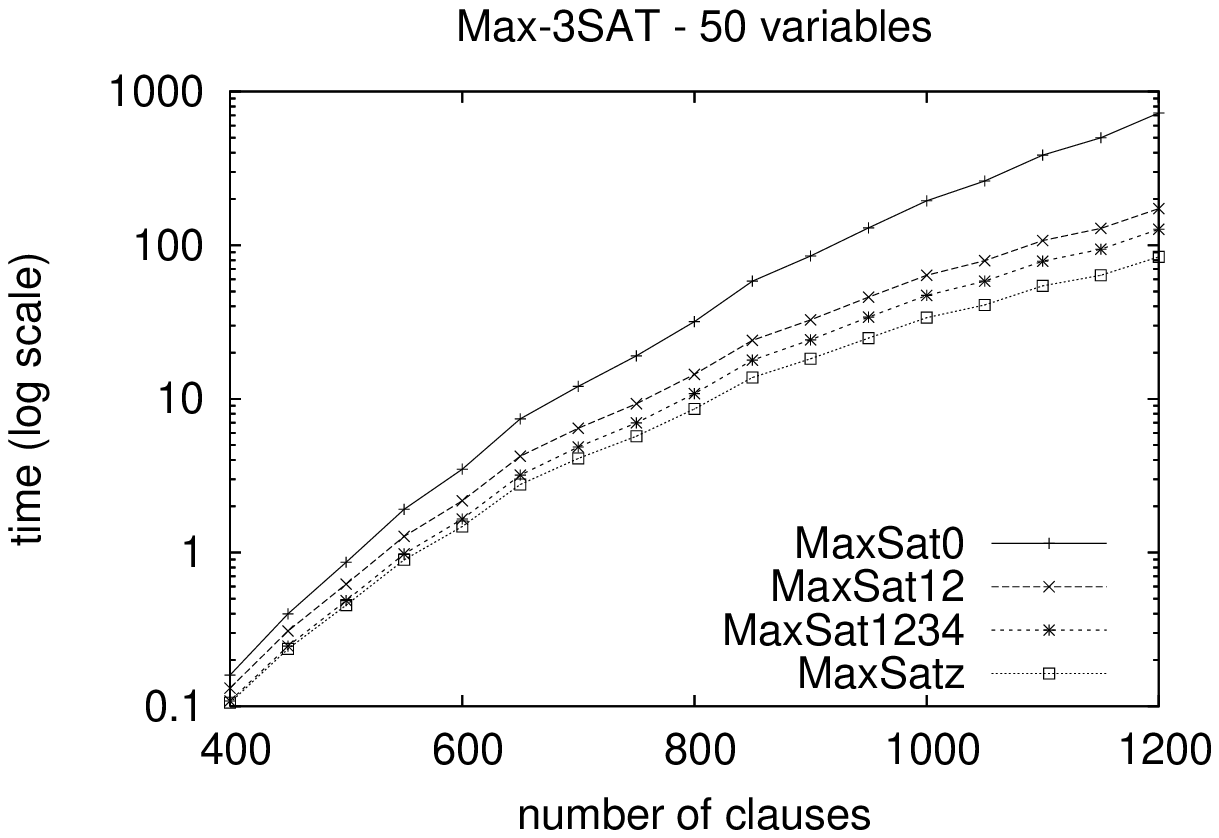} & \includegraphics[scale=.5]{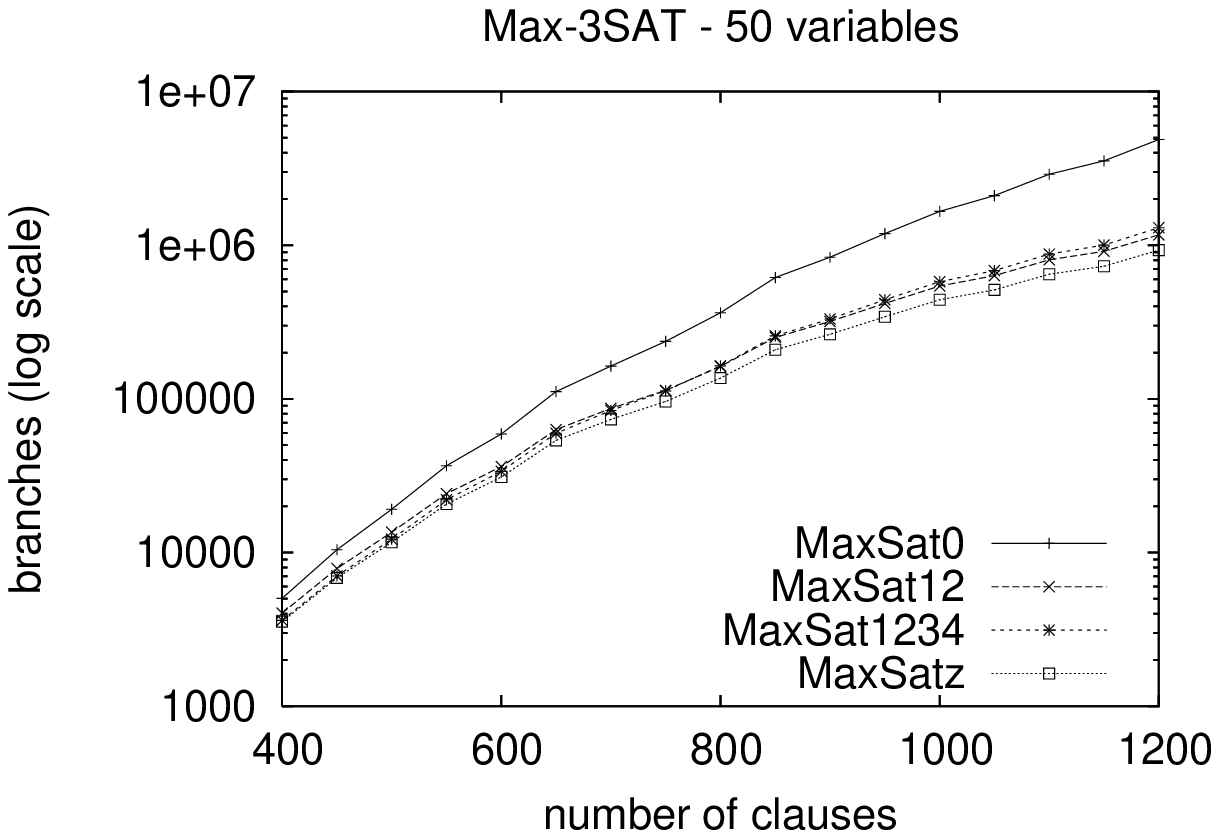} \\
\includegraphics[scale=.5]{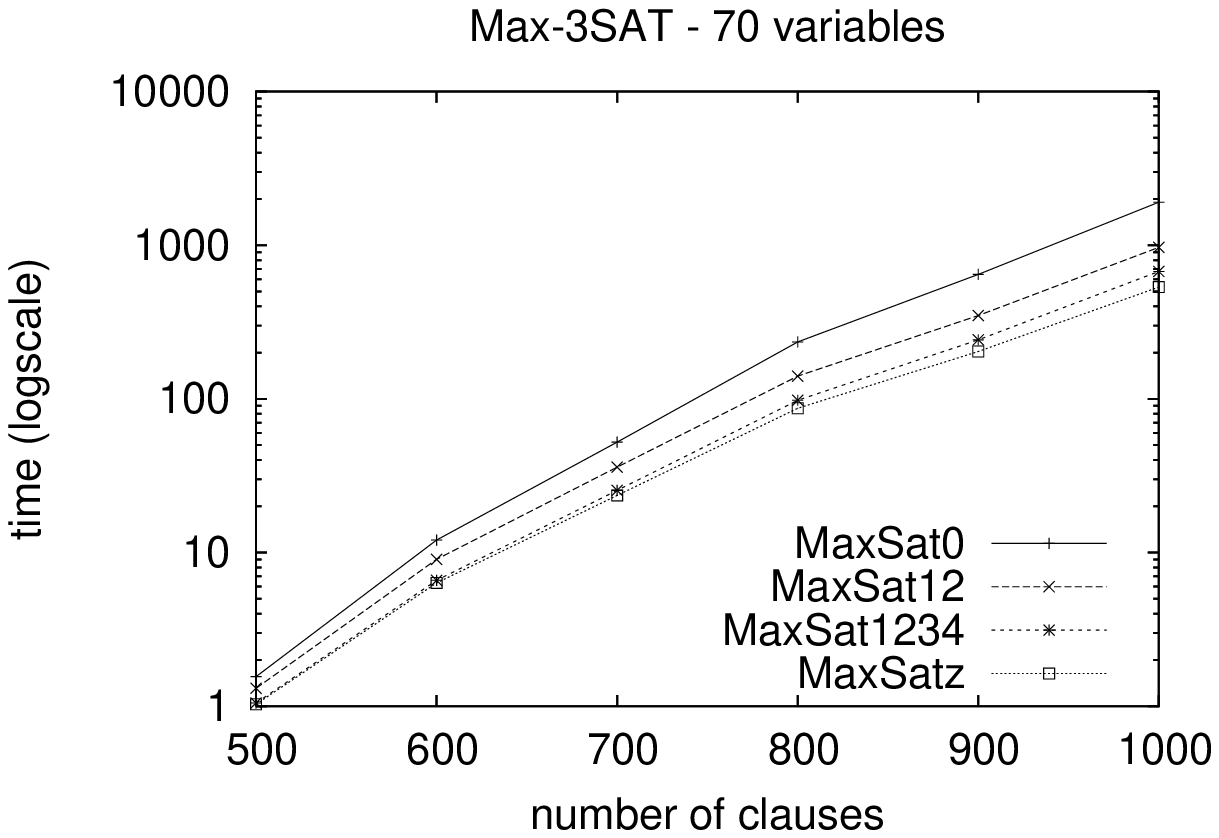} & \includegraphics[scale=.5]{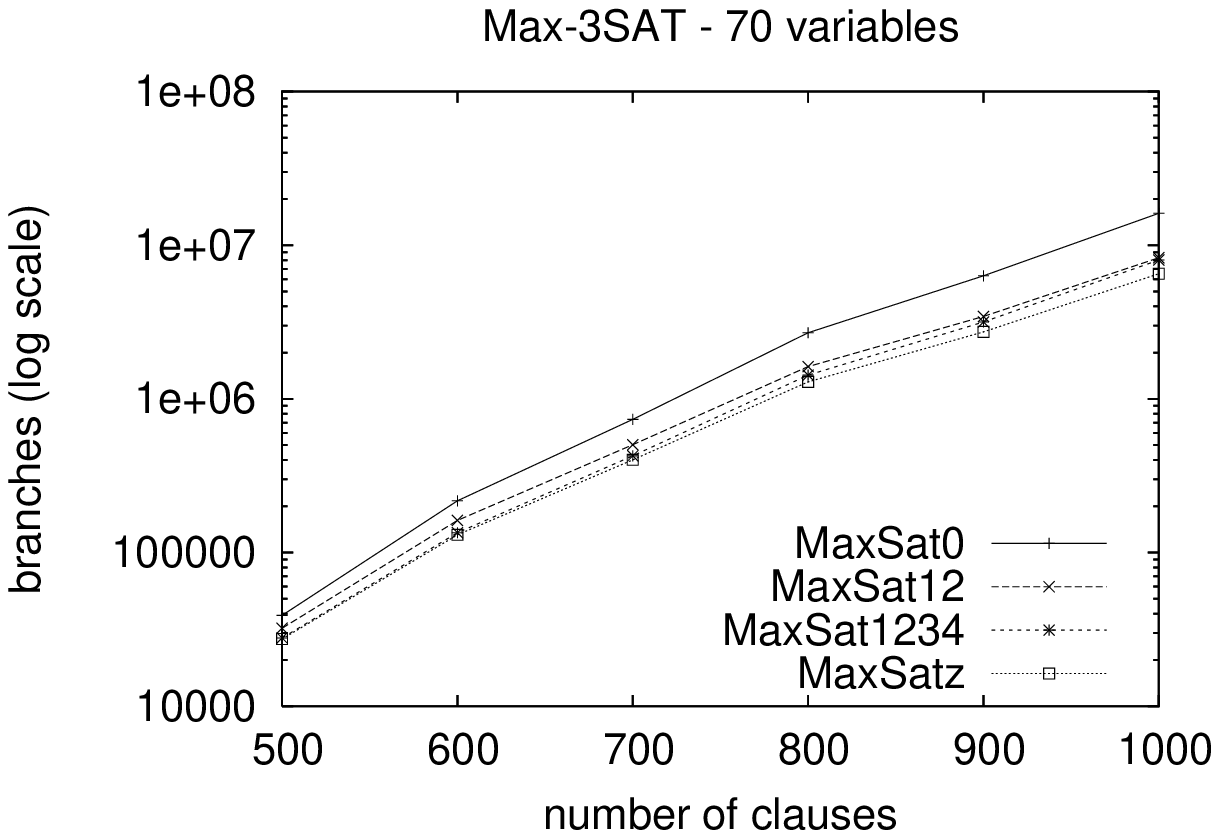} \\
\end{tabular}
\caption{\small Comparison among MaxSat12, MaxSat1234 and MaxSatz on random Max-3SAT instances.}\label{rule3sat}
\end{figure}

Although the rules do not involve ternary clauses, they are also powerful for Max-3SAT. Similarly to Max-2SAT, Rule~\ref{linear2} and Rule~\ref{linear3} slightly improve the lower bound when there 
are relatively few clauses, but do not improve the lower bound when the number 
of clauses increases. They improve the time performance thanks to the incremental lower bound computation
they allowed.
The gain increases as the number of clauses increases. For example, for problems with 70 variables, when the number of clauses is 600, MaxSat1234 is 36\% faster than MaxSat12 and, when the number of clauses is 1000, the gain is 44\%.
Rule~\ref{nonlinear1} and Rule~\ref{nonlinear2} improve both the lower bound and the time performance of MaxSatz. The gain increases as the number of clauses increases. 

In the third experiment we considered the Max-Cut problem for graphs with 50 vertices and a number of edges ranging
from 200 to 800. Figure~\ref{max-cut1} shows the results of comparing the inference rules on Max-Cut instances. 
We observe that the rules allow us to solve the instances much faster. Similarly to random Max-2SAT, Rule \ref{linear2} 
and Rule \ref{linear3} do not improve the lower bound when there are many clauses, but improve the 
time performance due to the incremental lower bound computation they allowed. Rule \ref{nonlinear1} and Rule \ref{nonlinear2} are more powerful than Rule \ref{linear2} 
and Rule \ref{linear3} for these instances, which only contain binary clauses but have some structure. In addition, the reduction of the tree size due to Rule \ref{nonlinear1} and Rule \ref{nonlinear2} contributes to the time performance of MaxSatz more than the incrementality of the lower bound computation, as for random Max-2SAT. For example, the search tree of MaxSatz 
for instances with 800 edges is 40 times smaller than the search tree of MaxSat1234, and MaxSatz is 47 times faster.

\vspace{-0.5cm}
\begin{figure}[H]
\begin{center}
\begin{tabular}{cc}
\includegraphics[scale=.5]{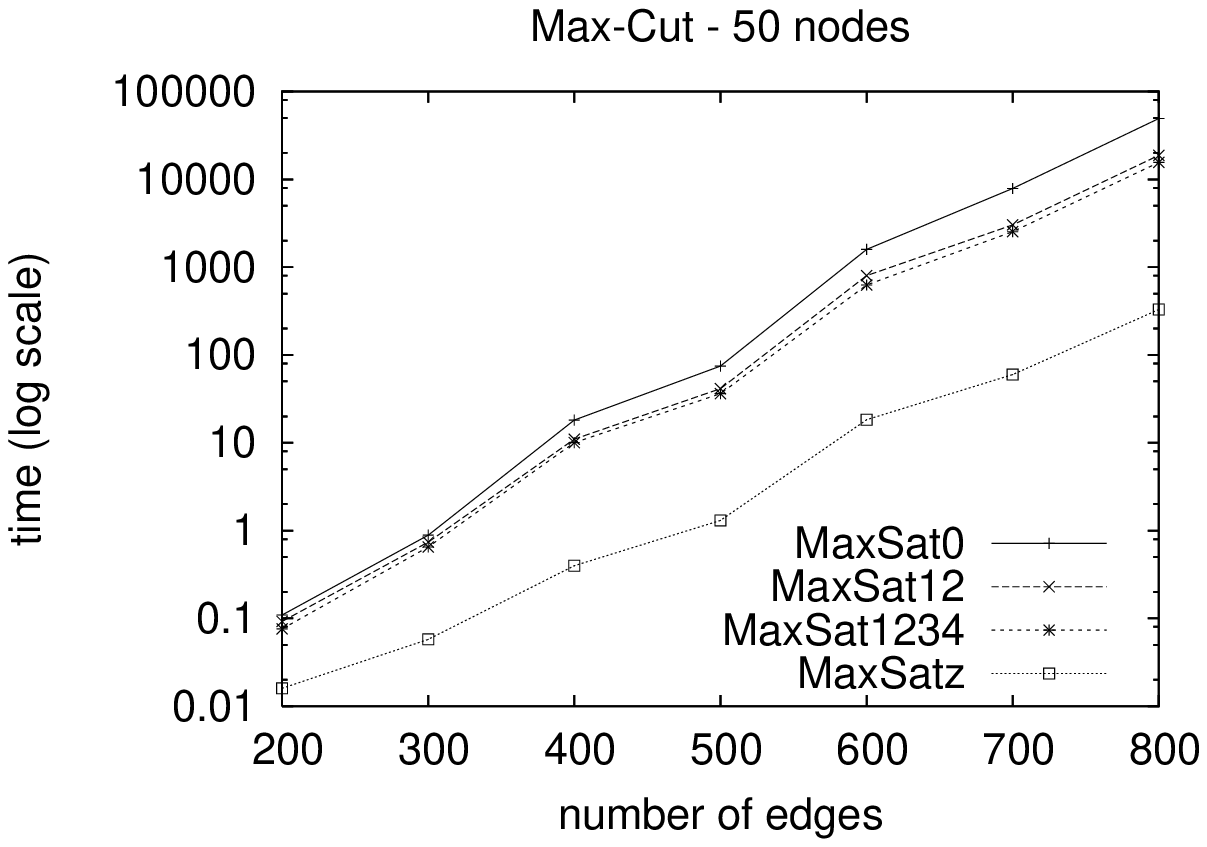} & 
\includegraphics[scale=.5]{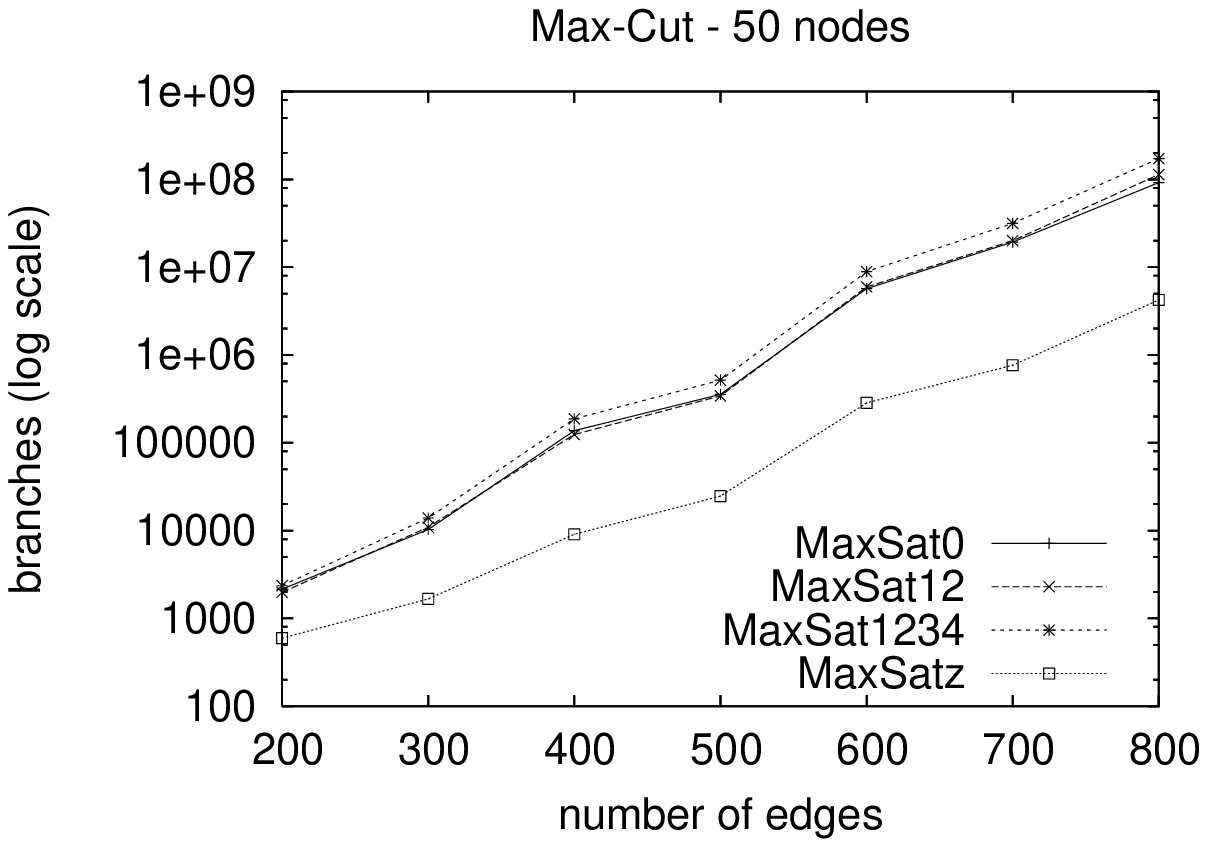}   \\
\end{tabular}
\end{center} 
\caption{\small Experimental results for Max-Cut}\label{max-cut1}
\end{figure}

In the fourth experiment we considered graph 3-coloring instances with 24 and 60 vertices, and with density of edges ranging
from 20\% to 90\%. Figure~\ref{coloring-cml} shows the results of comparing the inference rules on graph 3-coloring instances. We observe that Rule \ref{resolution} and Rule \ref{linear1} are not useful for these instances; the tree size of MaxSat0 and MaxSat12 is almost the same, and MaxSat12 is slower than MaxSat0. On the contrary, other rules are very useful for these instances, especially because they allow to reduce the search tree size by deriving better lower bounds.

\vspace{-0.5cm}
\begin{figure}[H]
\begin{center}
\begin{tabular}{cc}
\includegraphics[scale=.5]{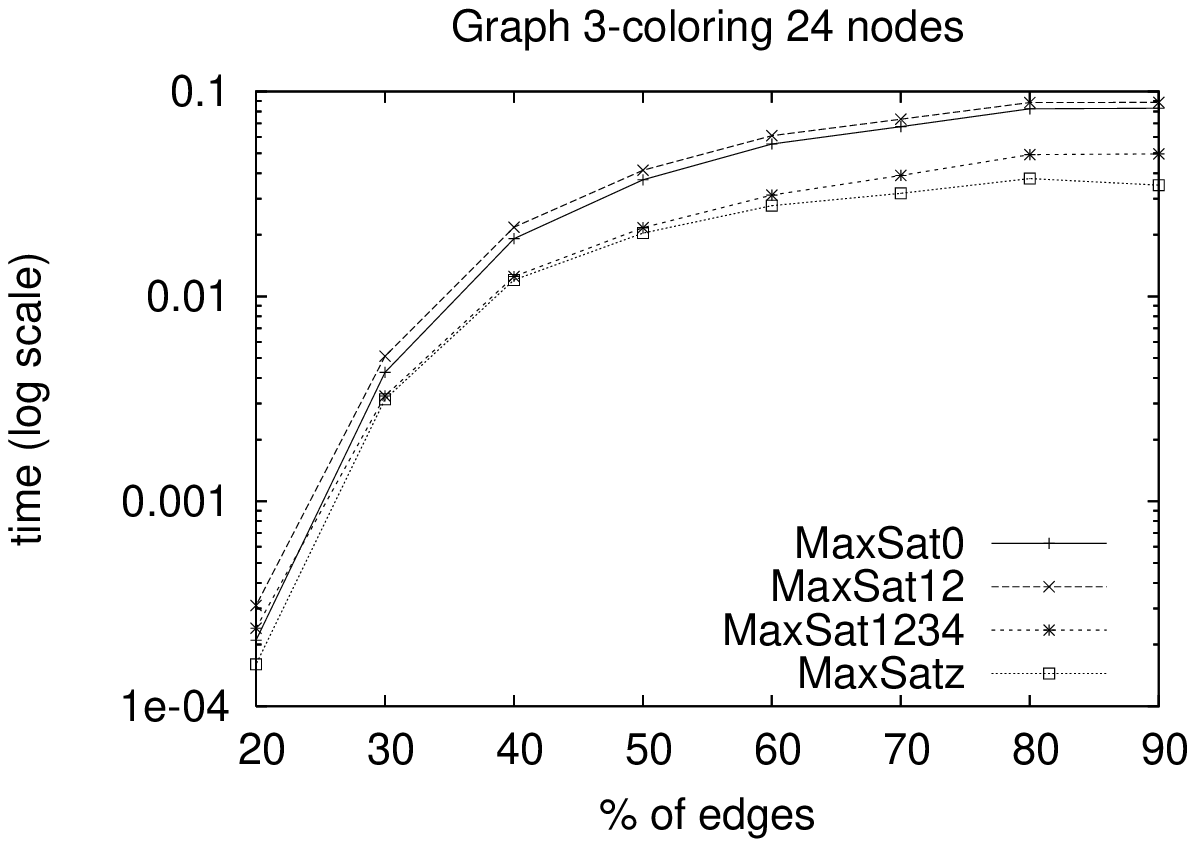} & \includegraphics[scale=.5]{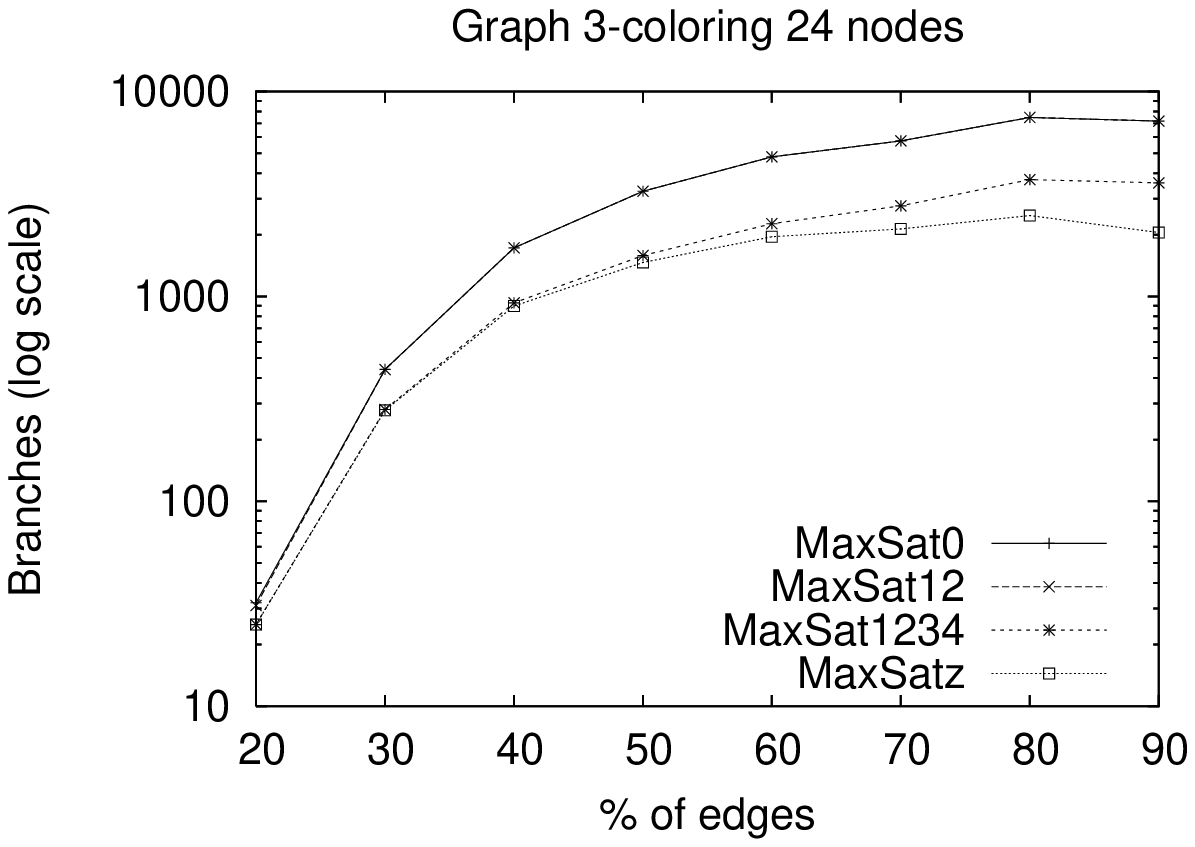}   \\
\includegraphics[scale=.5]{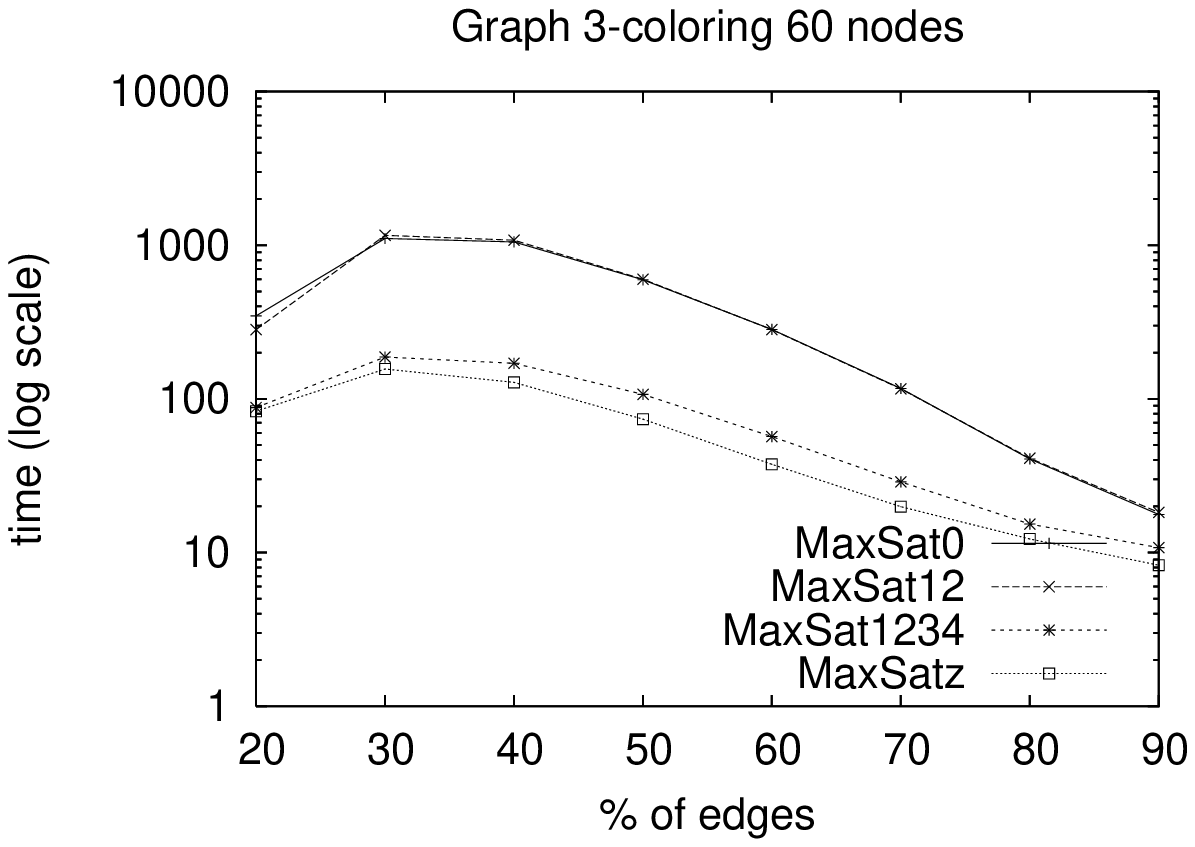} & \includegraphics[scale=.5]{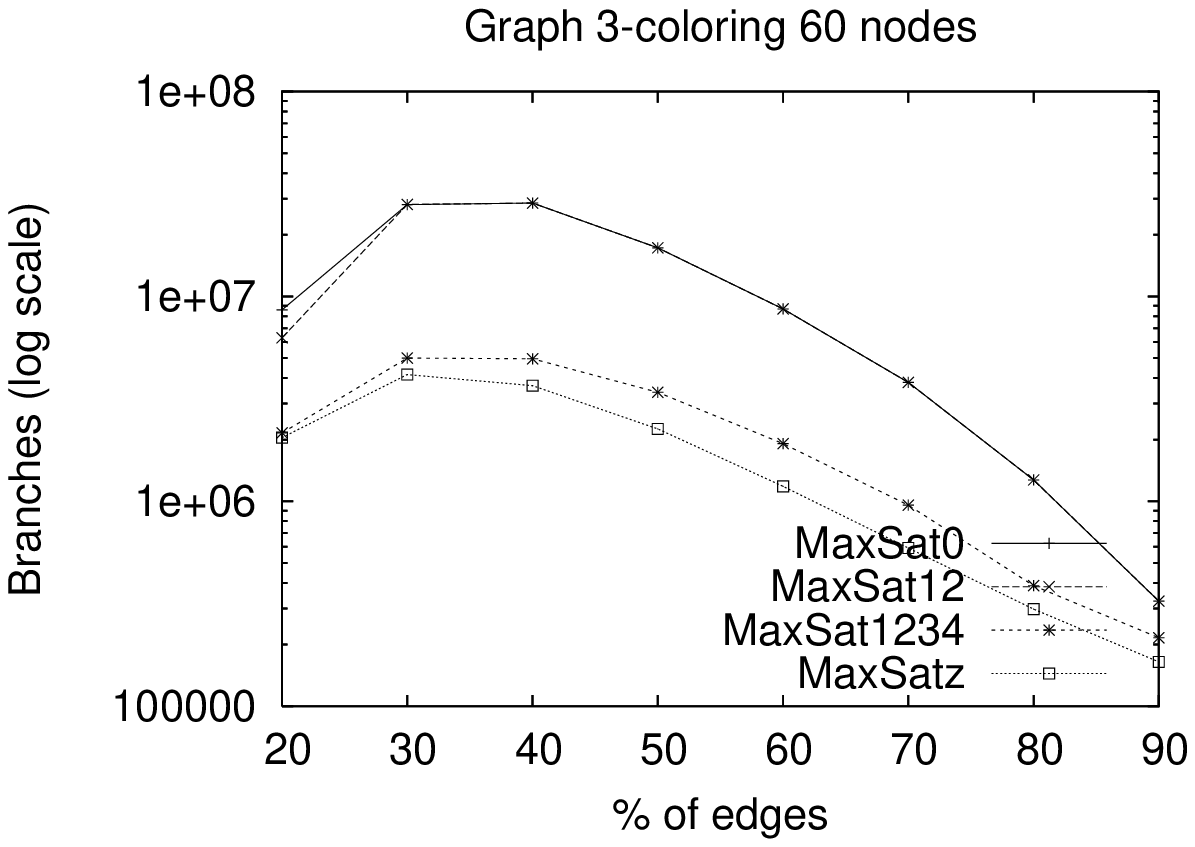}   \\
\end{tabular}
\end{center} 
\caption{\small Experimental results for Graph 3-Coloring}\label{coloring-cml}
\end{figure}

Note that Rule \ref{linear2} and Rule \ref{linear3} have more impact than Rule \ref{nonlinear1} and Rule \ref{nonlinear2} on reducing the cost of solving the instances. 
This is probably due to the fact that two unit clauses are needed to detect a contradiction, so that
Rule~\ref{linear2} and Rule~\ref{linear3} are applied many more times. Also note that the instances with 60 vertices become
easier to solve when the density of the graph is high.

In the fifth experiment, we compared different inference rules on the benchmarks submitted to the 
Max-SAT Evaluation 2006. Solvers ran in the same conditions as in the evaluation. In Table~\ref{tab:rule-evaluation},
the first column is the name of the benchmark set, the second column is the number of instances in the set, 
and the rest of columns display the average time, in seconds, needed by each solver to solve an instance
(the number of solved instances in brackets). The maximum time allowed to solve an instance was 30 minutes.

In is clear that MaxSat12 is better than MaxSat0, MaxSat1234 is better than MaxSat12, and MaxSatz is better than MaxSat1234. 
For example, MaxSatz solves three MAXCUT johnson instances within the time
limit, while the other solvers only solve two instances. The average time for MaxSatz to solve one of these three 
instances is 44.46 seconds, the third instance needing more time to be solved than the other two instances.

\begin{table} [H]
\scriptsize
\begin{tabular}{lccccc}
\multicolumn{1}{c}{Set Name} & \#Instances   & \texttt{MaxSat0}    & \texttt{MaxSat12} & \texttt{MaxSat1234} & \texttt{MaxSatz} \\
\hline
MAXCUT brock        &  12             &            471.01(10)       &           277.12(12)   &    225.11(12)   &     \bf 14.01(12) \\
MAXCUT c-fat          &   7       	   &  	      1.92 (5)	   &  	      3.11 (5)         &   2.84 (5)          &      \bf 0.07(5) \\
MAXCUT hamming &   6       	   &  	      39.42(2)	   &  	      29.43(2)        &     29.48(2)       &   \bf 171.30(3) \\
MAXCUT johnson   &   4       	   &  	      14.91(2)	   &  	      8.57 (2)         &   7.21 (2)          &     \bf 44.46(3) \\
MAXCUT keller        &   2       	   &  	      512.66(2)	   &  	      213.64(2)      &  163.26(2)       &       \bf 6.82(2) \\
MAXCUT p hat         &  12      	   &  	      72.16(9)	   &  	      286.09(12)   &   226.24(12)    &     \bf 16.81(12) \\
MAXCUT san           &  11      	   &  	      801.95(7)	   &  	      305.75(7)      &   245.70(7)      &    \bf 258.65(11) \\
MAXCUT sanr          &   4       	   &  	      323.67(3)	   &  	      134.74(3)      &     107.76(3)    &       \bf  71.00(4) \\
MAXCUT max cut    &  40      	   &  	      610.28(35)     &  	      481.48(40)    &       450.05(40)&        \bf 7.18(40) \\
MAXCUT SPINGLASS &   5         &  	      0.22 (2)	   &  	      0.19 (2)          &       0.15 (2)      &     \bf 0.14(2) \\
MAXONE                  &  45      	   &  	      \bf 0.03 (45) &  \bf 0.03 (45)    & \bf 0.03 (45)   &    \bf  0.03(45) \\
RAMSEY                  &  48      	   &  	      8.93 (34)	   &  	      8.42 (34)        &       7.80 (34)   &         \bf 7.78(34) \\
MAX2SAT 100VARS   &  50         &  	      95.01(50)	   &  	      11.30(50)       &      8.14 (50)    &    \bf  1.25(50) \\
MAX2SAT 140VARS   &  50         &  	      153.28(49)     &  	      51.76(50)       &        34.14(50) &       \bf  6.94(50) \\
MAX2SAT 60VARS &  50      	   &  	      1.35 (50)	   &  	      0.08 (50)         &        0.06 (50)  &       \bf 0.02(50) \\
MAX2SAT DISCARDED& 180     &  	      126.98(162)   &  	      71.85(173)     &       68.97(175)&      \bf 22.72(180) \\
MAX3SAT 40VARS  &  50      	   &  	      11.52(50)	   &  	      3.33 (50)         &       2.52 (50)   &       \bf 1.92(50) \\
MAX3SAT 60VARS  &  50      	   &  	      167.17(50)     &  	      72.72(50)        &        52.14(50) &      \bf 40.27(50) \\
\hline
\end{tabular}
\caption{Evaluation of the rules with benchmarks from the MAX-SAT Evaluation 2006.}\label{tab:rule-evaluation}
\end{table}

\subsection{Comparison of MaxSatz with Other Solvers}

In the first experiment, that we performed to compare MaxSatz with other state-of-the-art Max-SAT solvers,
we solved sets of 100 random Max-2SAT instances with 50, 100 and 150 variables; 
the number of clauses ranged from 400 to 4500 for 50 variables, from 400 to 1000 for 100 variables, and
from 300 to 650 for 150 variables. 
The results of solving such instances with BF, AGN, AMP, Lazy, toolbar, MaxSolver, UP and MaxSatz
are shown in Figure~\ref{2sat}. 
Along the horizontal axis is the number of clauses, and along the vertical axis
is the mean time, in seconds, needed to solve an instance of a set. When a solver needed too much time to solve the instances
at a point, it was stopped and the corresponding point is not shown in the figure. That is why for 50 variable instances, BF 
has only one point in the figure (for 400 clauses); and for 100 variable instances, BF and AMP also have only one point in the figure
(for 400 clauses).
The version of MaxSolver we used limits the number of clauses to 1000 in the instances to be solved. We ran it for instances
up to 1000 clauses.

We see dramatic differences
on performance between MaxSatz and the rest of solvers in Figure~\ref{2sat}. 
For the hardest instances, MaxSatz is up to two orders
of magnitude faster than the second best performing solvers (UP).  For those instances, MaxSatz needs 1 second to
solve an instance while solvers like MaxSolver and toolbar are not able to solve these instances after 10,000 seconds.

\begin{figure}
\begin{center}
\begin{tabular}{c}
\includegraphics[scale=.7]{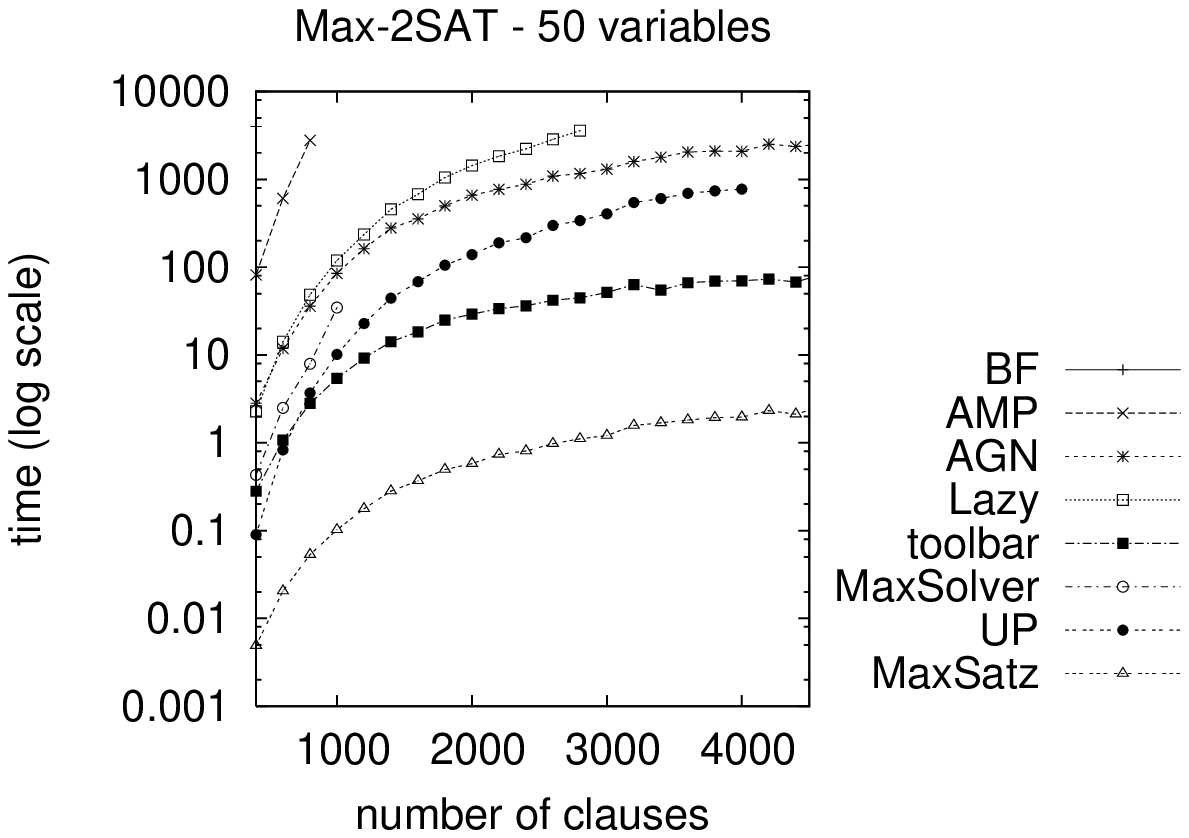} \\
\includegraphics[scale=.7]{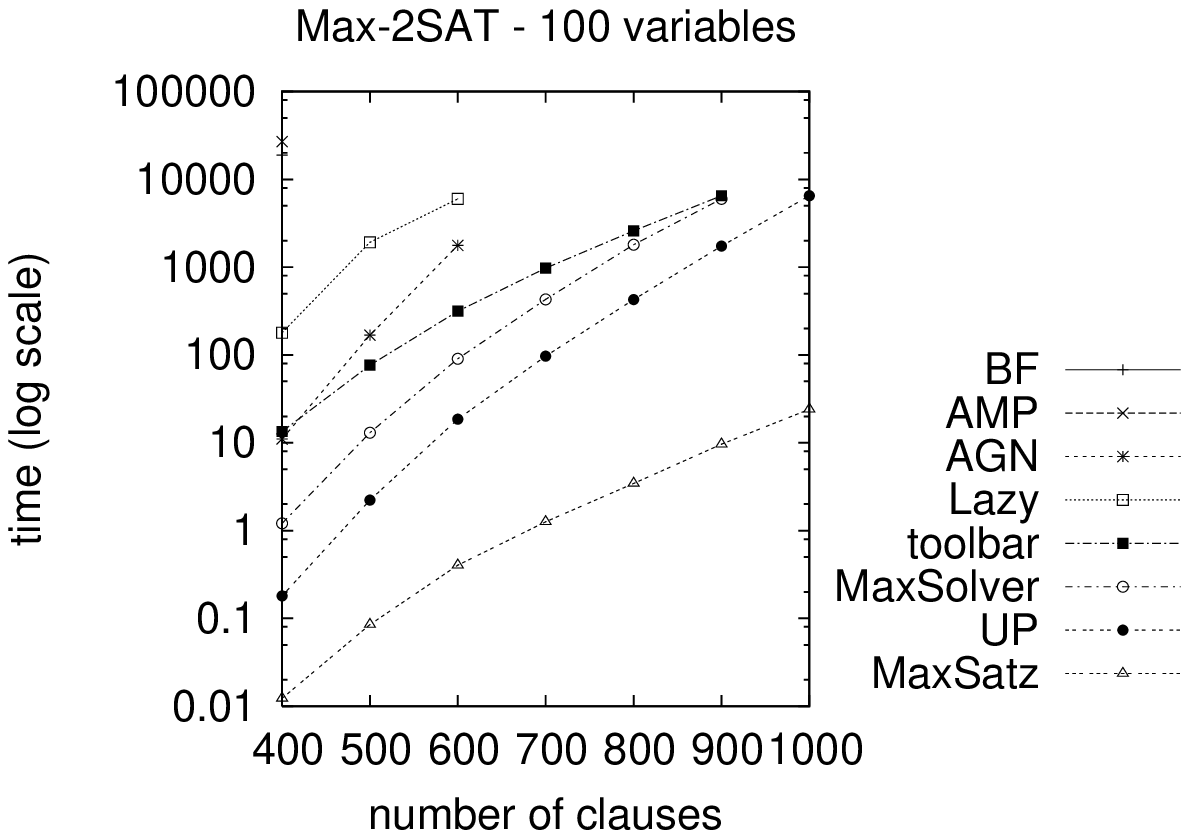} \\
\includegraphics[scale=.7]{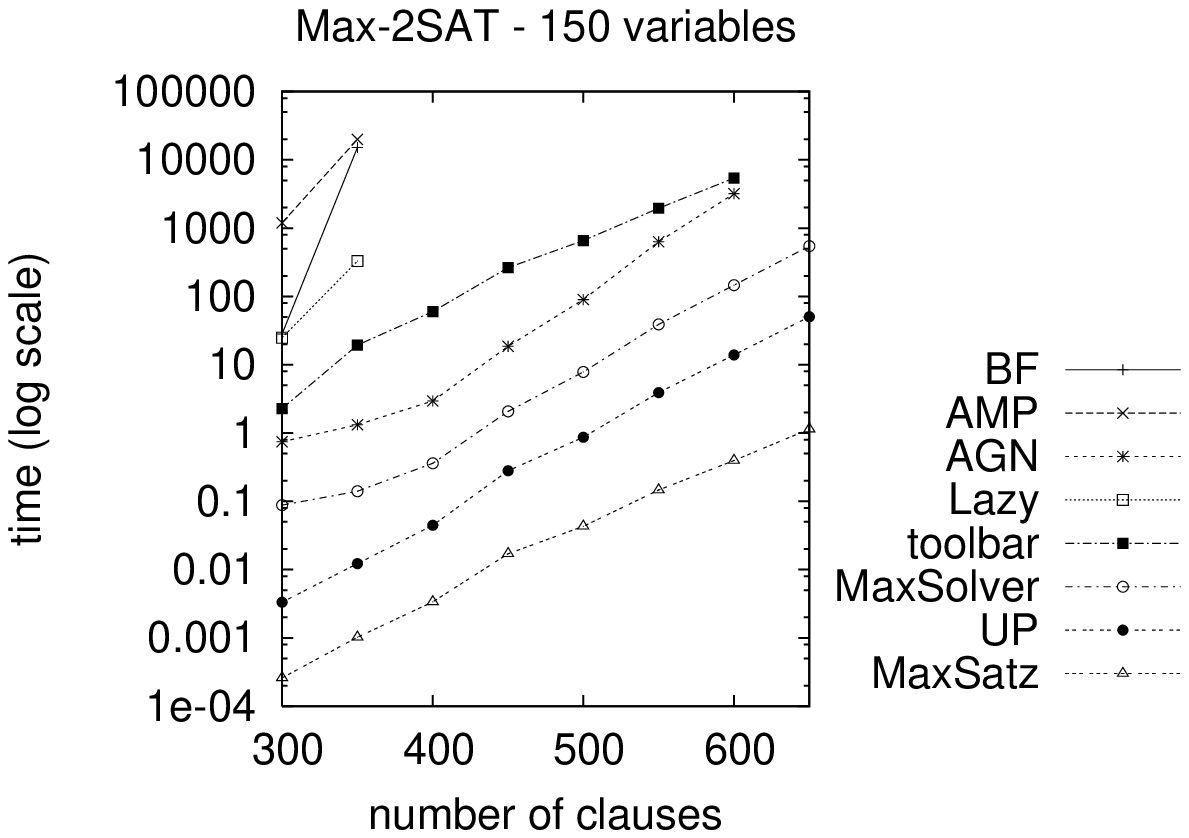}    \\
\end{tabular}
\end{center} 
\caption{Experimental results for 50-variable, 100-variable  and 150-variable random Max-2SAT  instances.}
\label{2sat}
\end{figure}

In the second experiment, we solved random Max-3SAT instances instead of random Max-2SAT instances.
The results obtained are shown in Figure~\ref{3sat}.
We did not consider AGN because it can only solve Max-2SAT instances.
We solved instances with 50, 70 and 100 variables; the number of clauses ranged from 500 to 1200 for 50 variables,
from 500 to 1000 for 70 variables,
and from 450 to 800 for 100 variables. For 70 variables, AMP has only one point in the figure (for 500 clauses) and BF is too slow. 
For 100 variables, we compared only the two best solvers.
Once again, we observe dramatic differences on the performance profile of MaxSatz and the rest of solvers.
Particularly remarkable are the differences between MaxSatz and toolbar (the second best performing solver on Max-3SAT),
where we see that MaxSatz is up to 1,000 times faster than toolbar on the hardest instances.

\begin{figure}
\begin{center}
\begin{tabular}{c}
\includegraphics[scale=.7]{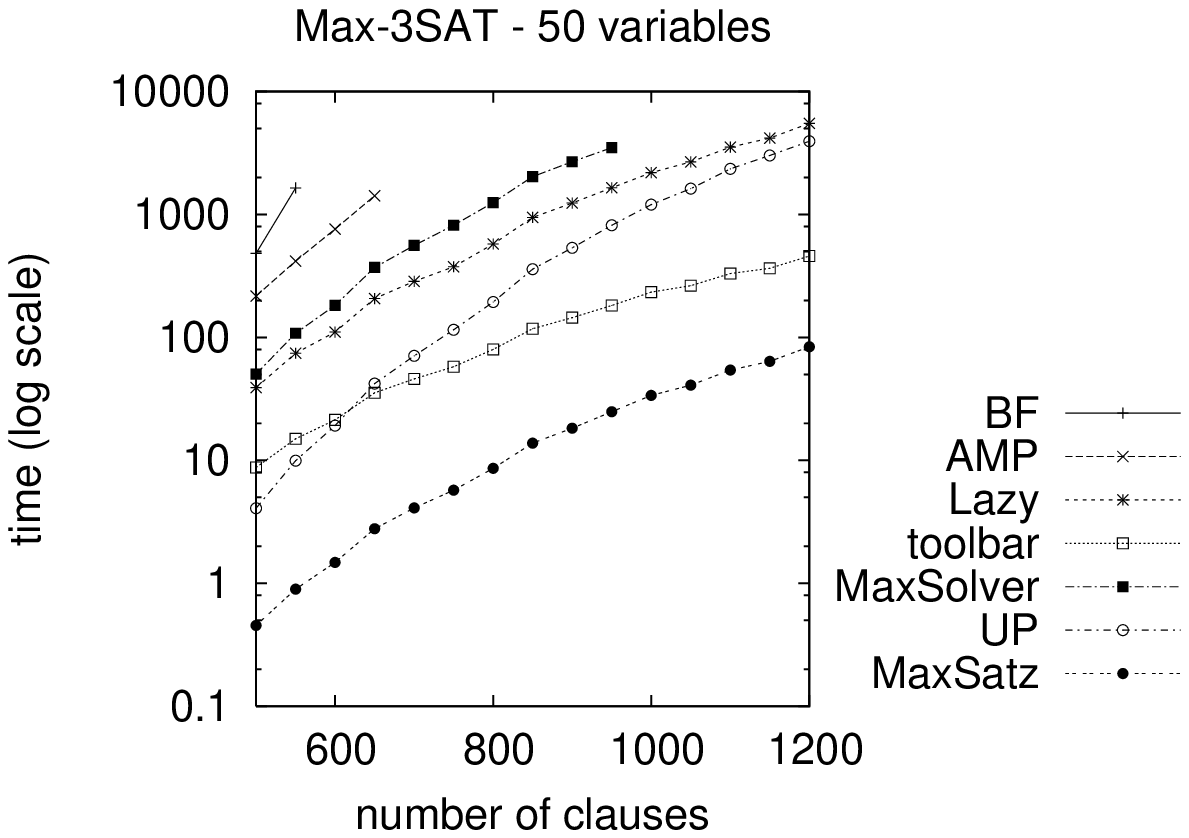} \\
\includegraphics[scale=.7]{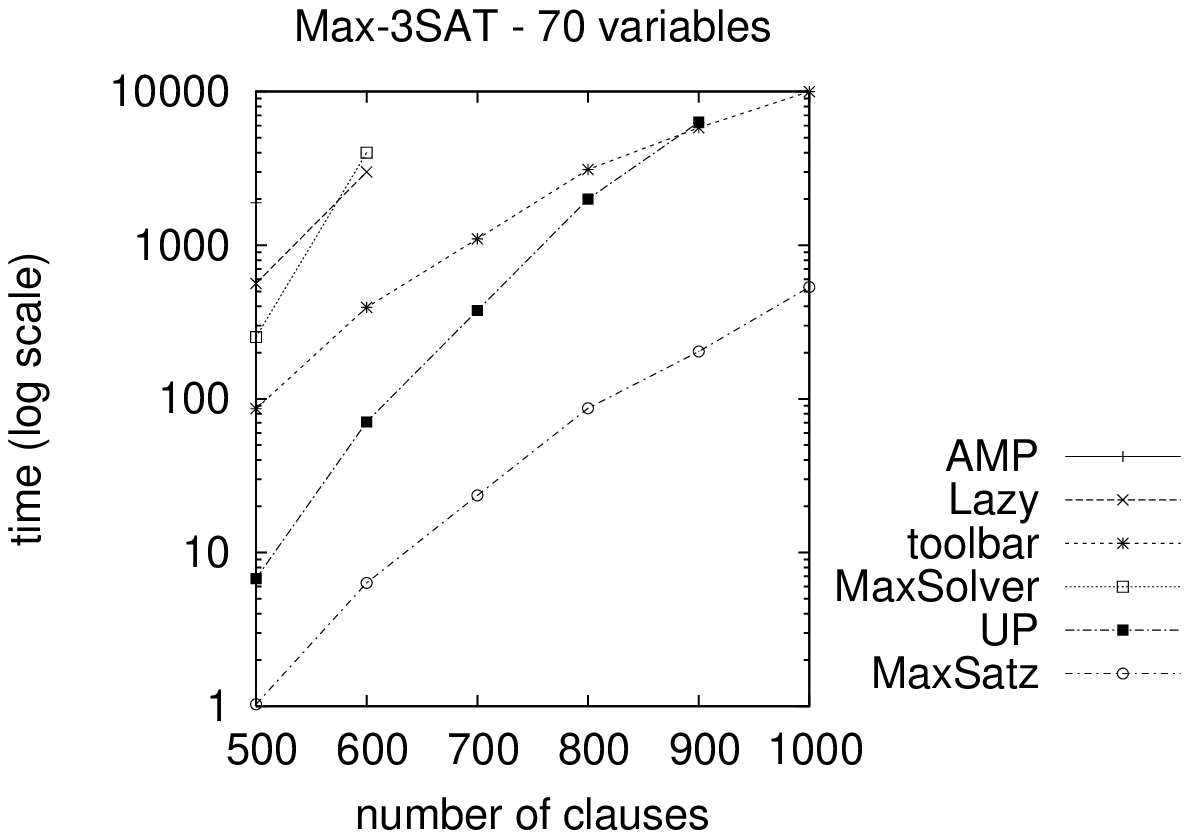} \\
\includegraphics[scale=.7]{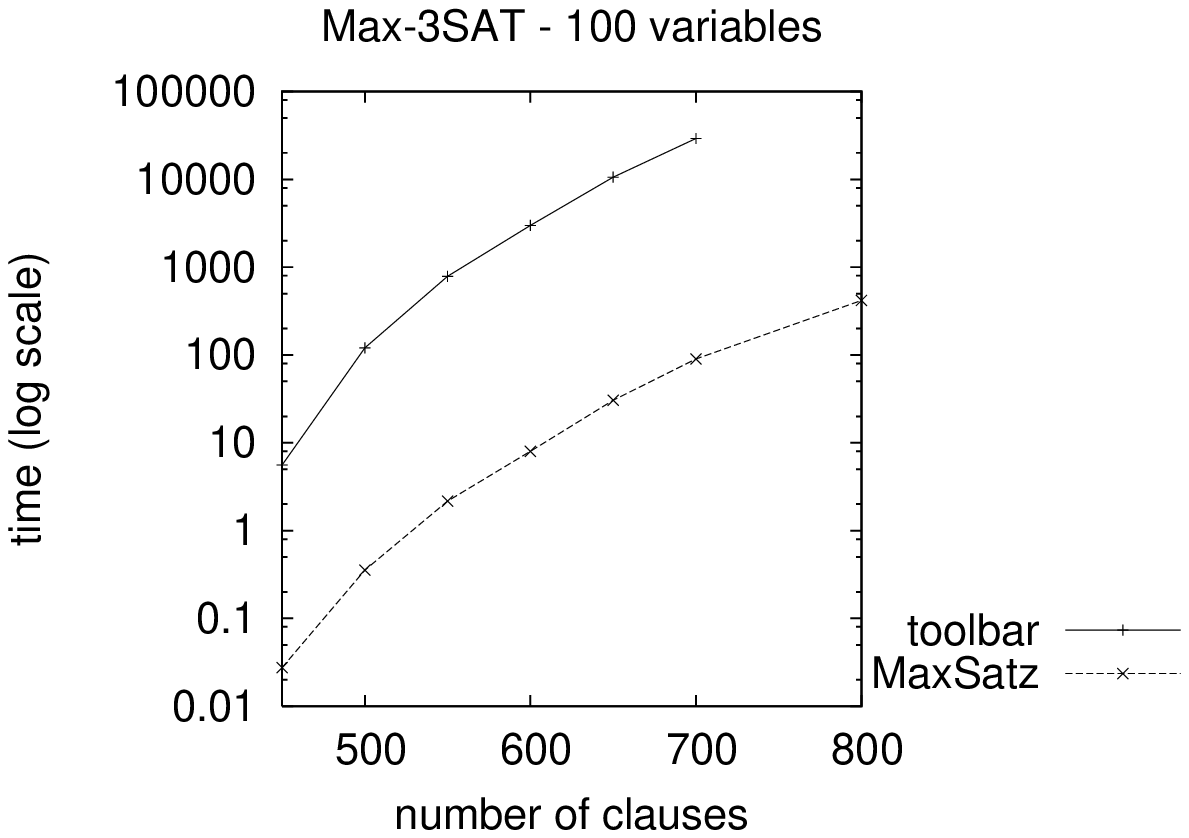}    \\
\end{tabular}
\end{center} 
\caption{Experimental results for 50-variable, 70-variable  and 100-variable random Max-3SAT instances.}\label{3sat}
\end{figure}

In the third experiment, we considered the Max-Cut problem of graphs with 50 vertices and a number of edges ranging
from 200 to 700. Figure~\ref{max-cut} shows the results obtained. BF has only one point in the figure (for 200 edges). 
MaxSolver solved instances up to 500 edges (1000 clauses).
We observe that MaxSatz is superior to the rest of solvers.

In the fourth experiment, we considered the 3-coloring problem of graphs with 24 and 60 vertices,
and a density of edges ranging
from 20\% to 90\%.  AGN was not considered because it can only solve Max-2SAT 
instances.
For 60 vertices, we only compared the three best solvers, of which MaxSolver is a different version 
not limiting the number of clauses of the instance to be solved.
Figure \ref{coloring} shows the comparative results for different solvers. MaxSatz is the best performing solver, 
and UP and MaxSolver are substantially better than the
rest of solvers.

\vspace{-0.5cm}
\begin{figure}[H]
\begin{center}
\begin{tabular}{cc}
\includegraphics[scale=.63]{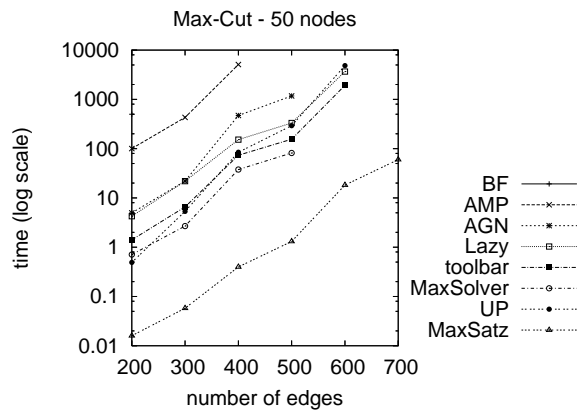} &    \\
\end{tabular}
\end{center} 
\vspace{-0.5cm}
\caption{\small Experimental results for Max-Cut}\label{max-cut}
\end{figure}
\vspace{-0.5cm}

\vspace{-0.99cm}
\begin{figure}[H]
\begin{center}
\begin{tabular}{c}
\includegraphics[scale=.63]{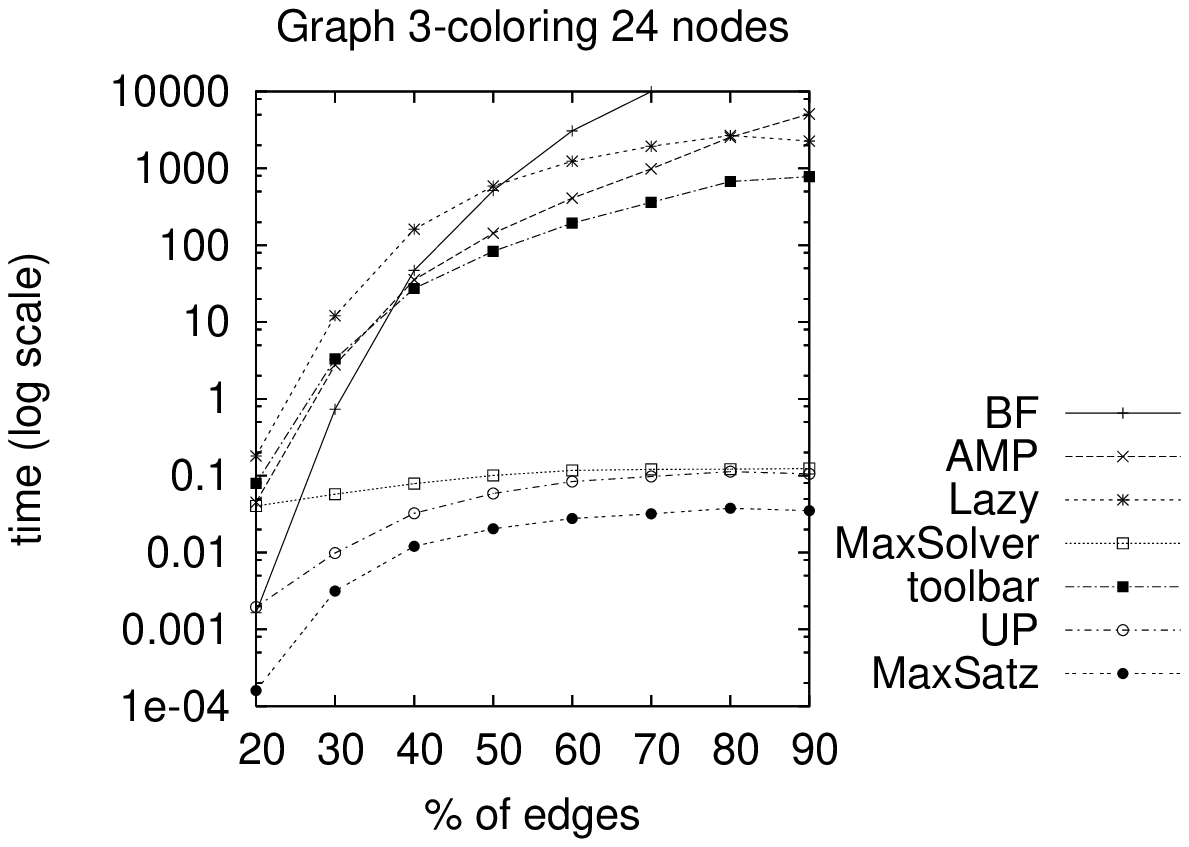} \\
\vspace{-0.2cm}
\includegraphics[scale=.63]{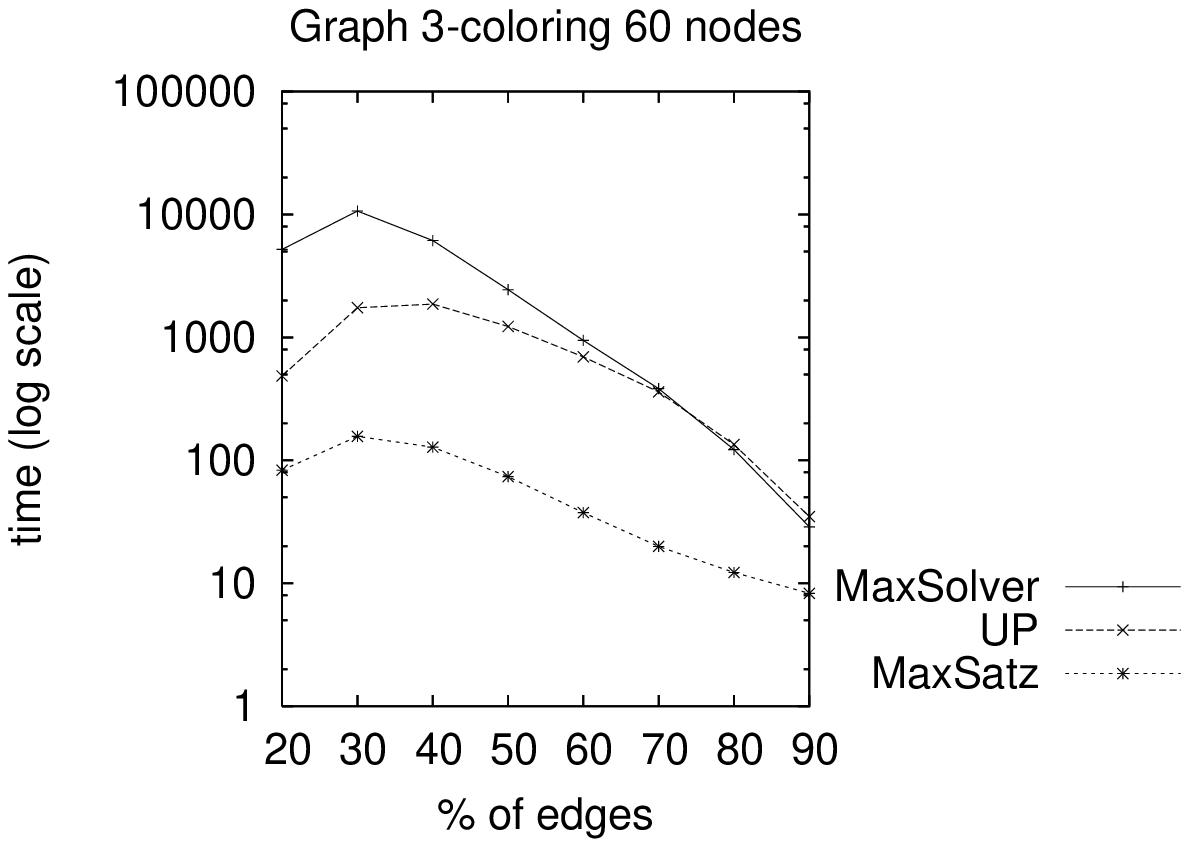} \\
\end{tabular}
\end{center} 
\vspace{-0.5cm}
\caption{\small Experimental results for Graph 3-Coloring}\label{coloring}
\end{figure}

In the fifth experiment, we compared the Max-SAT solvers on the benchmarks submitted to the Max-SAT Evaluation 2006. Solvers ran
in the same conditions as in the evaluation. In Table~\ref{tab:evaluation},
the first column is the name of the benchmark set, the second column is the number of instances of the set,
and the rest of columns display the average time, in seconds, needed by each solver to solve an instance within 
a time limit of 30 minutes (the number of instances solved within the time limit in 
brackets). A dash means that the corresponding solver cannot solve the set of instances.
It is clear that MaxSatz is the best performing solver for all the sets.

\begin{sidewaystable}
\scriptsize
\begin{tabular}{lcccccccccc}
\multicolumn{1}{c}{Set Name} & \#Instances &  \texttt{BF}   & \texttt{AMP}  & \texttt{AGN} & \texttt{toolbar} & \texttt{Lazy} & \texttt{MaxSolver} & \texttt{UP}         & \texttt{MaxSatz} \\
\hline
MAXCUT brock                          &   12               &       (0)          &   545.81(1)    &856.65(8)     & 470.23(12)     & 159.28
(12)   &      380.09(2)                 &      629.85(9)    &  \bf 14.01(12)    \\
MAXCUT c-fat                            &    7                 & 6.06 (1)       &   1.95 (3)       &32.70(5)        & 42.84(5)          & 13.23 (4)     & 41.58(3)                &   7.19 (5)          & \bf 0.07(5)      \\
MAXCUT hamming                   &    6                 &      (0)           &   636.04(1)    &159.99(1)     & 145.84(2)       & 265.35 (2)   &      (0)                      &       294.89(2)   & \bf 171.30(3)     \\
MAXCUT johnson                     &    4                 &      (0)           &   394.17(2)    &92.90(2)       & 11.07(2)          & 13.50 (2)     & 1.34 (1)                   &      29.42(2)      & \bf 44.46(3)      \\
MAXCUT keller                          &    2                 &      (0)           &   197.15(1)    &39.36(1)       & 255.39(2)        & 348.75 (2)  &      (0)                      &      615.54(2)    & \bf 6.82 (2)      \\
MAXCUT DIMACS p hat          &12                   & 605.44(2)   &   107.79(8)    &16.11(8)       & 235.60(11)      & 259.33 (10)& 14.00(8)                 &      140.23(9)    & \bf 16.81(12)     \\
MAXCUT san                             &   11                &      (0)           &   563.19(1)   &72.35(2)        & 568.09(7)        & 956.54 (5)  & 283.34(2)               &      812.47(5)    & \bf 258.65(11)    \\
MAXCUT sanr                            &    4                 &      (0)           &   428.18(1)   &909.32(3)      & 234.89(3)        & 410.53 (3)  & 138.32(1)              &      538.10(3)    & \bf 71.00(4)      \\
MAXCUT max cut                      &   40                &      (0)           &         (0)           &1742.79(3)   & 736.34(18)     & 1027.21 (7)& (0)                          &      623.03(13)   & \bf 7.18(40)   \\
MAXCUT SPINGLASS             &    5                  & 0.21 (1)      &   0.13 (1)       &   12.70(2)      & 5.72 (2)            & 0.05 (1)       & 570.68(2)              &   0.86 (2)           & \bf 0.14(2)      \\
MAXONE                                    &   45                 & 0.02 (21)    &   \bf 0.03 (45)&     -                & 35.35(44)         & 278.58 (26)& 0.06 (45)               &      0.31 (45)      & \bf 0.03 (45)     \\
RAMSEY ram k                         &   48                 & 8.53 (30)    &   38.44(30)    &     -                 & 4.14(27)           & 10.48 (25)  & 0.20 (20)                &      19.65(25)    & \bf 7.78 (34)     \\
MAX2SAT 100VARS               &   50                 & 0.14 (10)    &   143.23(11)  &185.69(30)   & 244.05(34)       & 273.44 (22)& 532.47(16)             &      192.34(48) & \bf 1.25 (50) \\
MAX2SAT 140VARS               &   50                 & 0.08 (10)    &   91.93(12)    &126.34(28)   & 262.30(26)       & 217.12 (17)& 168.42(18)             &      75.57(39)    & \bf 6.94 (50)  \\
MAX2SAT 60VARS                 &   50                  & 1.92 (3)      &   514.02(44) &6.34 (50)       & 2.01 (50)           & 26.44 (50)  & 81.82(50)                &      0.94 (50)     & \bf 0.02 (50)     \\
MAX2SAT DISCARDED         &  180                & 357.65(28)&   439.54(76) &99.70(108)    & 178.23(116)     & 85.08 (87)  & 308.58(73)             &      166.29(149)& \bf 22.72(180) \\
MAX3SAT 40VARS                 &   50                  & 170.49(22)&   202.18(50) &     -                  & 10.19 (50)        & 69.72 (50)  & 66.34(49)               &      60.50(50)     & \bf 1.92(50)     \\
MAX3SAT 60VARS                 &   50                  & 4.07 (16)    &   168.00(25) &     -                  & 361.95(43)      & 242.40 (28) & 139.03(22)            &      166.76(37)   & \bf 40.27(50)     \\
\hline
\end{tabular}
\caption{Experimental results for benchmarks from the MAX-SAT Evaluation 2006.}\label{tab:evaluation}
\end{sidewaystable}

\section{Related Work} \label{related-work}

The simplest  method to compute a lower bound consists of just counting the number of clauses
unsatisfied by the current partial assignment \cite{BF99}. One step forward is to incorporate an underestimation of the number of
clauses that will become unsatisfied if the current partial assignment
is extended to a complete assignment. The most basic method was defined by Wallace and Freuder~\citeyear{WF96}:

$$\texttt{LB}(\phi) = \#emptyClauses(\phi) + \sum_{x \mbox{ occurs in } \phi} min (ic(x),ic(\bar x))$$ 

\noindent where $\phi$ is the CNF formula associated with the current partial assignment, and
$ic(x)$ ($ic(\bar x)$) ---inconsistency count of $x$ ($\bar x$)--- is the number of unit clauses of $\phi$ that contain $\bar x$ ($x$).

The underestimation of the lower bound can be improved by applying to binary clauses
the Directional Arc Consistency (DAC) count defined by Wallace~\citeyear{Wallace95} for Max-CSP.
The DAC count of a value of the variable $x$ in $\phi$ is the number of variables which are inconsistent 
with that value of $x$. For example, if $\phi$
contains clauses $x\vee y$, $x \vee \bar y$, and $\bar x \vee y$, the value 0 of $x$
is inconsistent with $y$.
Note that value 0 of $y$ is also
inconsistent with $x$. These two inconsistencies are not disjoint and cannot be summed. Wallace defined a
direction from $x$ to $y$, so that only the inconsistency for value 0 of $x$ is counted.
After defining a direction between every pair of variables sharing a constraint,
one computes the DAC count for all values of $x$ by checking all variables to which a direction from
$x$ is defined. The underestimation considering 
the DAC count of Wallace is as follows:

$$\sum_{x \mbox{ occurs in } \phi} (min (ic(x),ic(\bar x))+min(dac(x), dac(\bar x))$$

\noindent where $dac(x)$ ($dac(\bar x)$) is the DAC count of the value 1(0) of $x$.
\citeauthor{Wallace95} statically defined all directions,
so that $dac(x)$ and $dac(\bar x)$ can be computed
in a preprocessing step for every $x$ and do not need to be recomputed during search. 
This is improved by Larrosa, Meseguer and Schiex~\citeyear{LMS99} by introducing reversible DAC, which searches for better directions to obtain a better
lower bound at every node of the search tree. An improvement of DAC counts is the additional incorporation of inconsistencies
contributed by disjoint subsets of variables, based on particular variable partitions~\cite{LM02a}.

Inconsistent and DAC counts deal with unit and binary clauses. Lower bounds dealing with longer
clauses include star rule~\cite{SZ04,AMP04} and UP~\cite{LMP05}.

In the star rule, the underestimation of the lower bound is the number of disjoint
inconsistent subformulas of the form $\{l_1, \ldots, l_k, \bar l_1 \vee \cdots \vee \bar l_k\}$.
The star rule, when $k=1$, is equivalent to the inconsistency counts of Wallace and Freuder.

UP subsumes the inconsistent count method based on unit clauses and the star rule. Its effectiveness for
producing a good lower bound can be illustrated with the following example:
let $\phi$ be a CNF formula containing the clauses $x_1, \bar x_1 \vee x_2, \bar x_1 \vee x_3, \bar x_2 \vee \bar x_3 \vee x_4,
x_5, \bar x_5 \vee x_6, \bar x_5 \vee x_7, \bar x_6 \vee \bar x_7 \vee \bar x_4$.
UP easily detects that inconsistent subset with 8 clauses and 7 variables, in time linear in the size of the formula. 
Note that this subset is not detected by any of the lower bounds described above, except for the variable partition based
approach of Larrosa and Meseguer~\citeyear{LM02a} in the case that the 7 variables are in the same partition.

We mention two more lower bound computation methods. One is called LB4  and 
was defined by Shen and Zhang~\citeyear{SZ04}. It is similar to UP but restricted to Max-2SAT instances and using a static
variable ordering. Another is based on linear programming and was defined by Xing and Zhang~\citeyear{XZ05}.

Regin et al.~\citeyear{RPBP01} suggested to use arc consistency, instead of unit propagation, to detect disjoint inconsistent subsets of
constraints in weighted constraint networks. However, to the best of our knowledge, this idea has not been incorporated
in any lower bound computation method implemented by the Constraint Programming community.

A good lower bound computation method has a dramatic
impact on the  performance of a Max-SAT solver. Another approach to speed up a Max-SAT solver
consists of applying inference rules to transform a Max-SAT instance $\phi$ into an equivalent but simpler 
Max-SAT 
instance $\phi'$.  Inference rules that have proven to be useful in practice include: (i)~the pure literal rule~\cite{AMP03,XZ04,LMP05,ZSM03a}; 
(ii)~the dominating unit clause rule, first proposed by Niedermeier and Rossmanith~\citeyear{NR00},
 and later applied in several solvers~\cite{AMP04,XZ04,LMP05}; 
(iii)~the almost common clause rule, first proposed by Bansal and Raman~\citeyear{BR99} and restated as Rule~\ref{resolution} in this paper. 
The rule was extended to weighted Max-SAT by\citeauthor{AMP04}~\citeyear{AMP04};
was called neighborhood resolution by~\citeauthor{LH05a}~\citeyear{LH05a}; and used as a preprocessing technique by~\citeauthor{AMP04}~\citeyear{AMP04}, \citeauthor{SZ05}~\citeyear{SZ05}, and \citeauthor{LMP05}~\citeyear{LMP05};
(iv)~the complementary unit clause rule~\cite{NR00}, restated as Rule \ref{linear1} in this paper; and (v)~the coefficient-determining unit propagation rule~\cite{XZ05}
based on integer programming.

The inference rules presented in this paper simplify a Max-SAT formula $\phi$ and allow to improve the lower bound computation, since they all
transform a Max-SAT formula $\phi$ into a simpler and equivalent formula containing more empty clauses. 
Their soundness (i.e., the fact that they transform
a formula into an equivalent one)  can be proved in several ways, including (i) checking all possible variable assignments, (ii) using integer programming as done in Section \ref{inference-rules}, and 
(iii) using soft local consistency techniques defined for Weighted Constraint Networks (WCN);
Max-SAT can be defined as a subcase of WCN where variables are Boolean and only unit costs are used.

Soft local consistency techniques for WCN are based on two basic equivalence preserving transformations 
called $projection$ and $extension$
\cite{Schiex00,CS04}. Given a Max-SAT instance, projection replaces two binary clauses $x\vee y$ and $x \vee \bar y$ with
the unit clause $x$, which is Rule \ref{resolution} for $k$=2. Extension is the inverse operation of projection
and replaces a unit clause $x$ with two binary clauses $x\vee y$ and $x \vee \bar y$ for a selected variable $y$. If the projection operation 
is rather straightforward for a SAT or Max-SAT instance, the extension operation is very ingenious. To see this,
note that Rule \ref{linear2} can be proved or applied with an extension followed by a projection:

\begin{eqnarray} 
l_1, ~\bar{l}_1 \vee \bar{l}_2, ~l_2 & = & l_1 \vee l_2, l_1 \vee \bar{l}_2, ~\bar{l}_1 \vee \bar{l}_2, ~l_2 \nonumber\\ 
                                                              & = &  l_1 \vee l_2, \bar{l}_2, l_2  \nonumber\\ 
                                                              & = & l_1 \vee l_2, \Box \nonumber
\end{eqnarray}

Lemma \ref{lemma1} can also be proved using an extension followed by  a projection:

\begin{eqnarray} 
l_1,~ \bar{l}_1 \vee l_2 & = & l_1\vee \bar{l}_2, ~l_1\vee l_2, ~\bar{l}_1 \vee l_2 \nonumber\\ 
                                            & = & l_1\vee \bar{l}_2, l_2  \nonumber
\end{eqnarray}

The extension operation cannot be used in an unguided way because it may cancel a previous projection. One way
to guide its use is to define an ordering between variables to enforce directional arc consistency~\cite{Cooper03,CS04}. 
Directional arc consistency allows to concentrate weights on the same variables by shifting weights from earlier variables to later ones in a given ordering. For example
 if  $x_1 < x_2$ in a given variable ordering, one can extend 
unit clause $x_1$ to  $x_1 \vee x_2, x_1 \vee \bar{x}_2$, but cannot extend unit clause $x_2$ to $x_1\vee x_2, ~\bar{x}_1 \vee x_2$, allowing unit clauses to be concentrated on variable $x_2$. Nevertheless, how to define the variable ordering to efficiently exploit as much as possible the power of soft arc consistency techniques in the lower bound computation remains an open problem.

The approach with inference rules for Max-SAT presented in this paper does not need any predefined ordering among
variables, since rule applications combining several projection and extension operations are entirely guided by unit
propagation.

The projection and extension operations can be extended to constraints involving more than two variables to achieve high-order consistency in WCN~\cite{Cooper05}. For a Max-SAT instance, the extended projection and extension operations can be stated using 
Rule~\ref{resolution} for $k$$>$2. For the two formulas $\phi_1$ and $\phi_2$ in Rule~\ref{resolution}, replacing $\phi_1$ with $\phi_2$ is a projection and  $\phi_2$ with
$\phi_1$ is an extension. Given a unit clause $x$ and three variables $x$, $y$, $z$, the extension of the unit clause $x$
to the set of three variables can be done as follows : replacing $x$ by $x\vee y$ and $x \vee \bar{y}$, and then $x\vee y$ and $x \vee \bar{y}$ by $x\vee y \vee z$, $x\vee y \vee \bar{z}$, $x \vee \bar{y} \vee z$ and $x \vee \bar{y} \vee \bar{z}$.

Rule~\ref{nonlinear1} can be proved or applied by extending the four clauses of $\phi_1$ to ternary clauses on the three variables 
of $l_1$, $l_2$ 
and $l_3$, and then applying the projection operation to obtain~$\phi_2$.

Larrosa et al.~\citeyear{Larrosa06}, based on a logical approach, independently and in parallel with our work, defined and implemented a chain resolution rule and a cycle resolution 
rule for weighted Max-SAT. These two rules are extensions of Rules 2-RES and 3-RES presented, also independently and in parallel with our work~\cite{HL06}.

The chain resolution could be stated as follows:

$$
\left\{
\begin{array}{l}
(l_1, u_1),\\
({\bar l}_i \vee l_{i+1}, u_{i+1})_{1\leq i <k},\\
({\bar l}_k, u_{k+1})
\end{array}
\right\}
 =  
\left\{
\begin{array}{l}
(l_i, m_i - m_{i+1})_{1 \leq i \leq k}, \\
({\bar l}_i \vee l_{i+1}, u_{i+1} - m_{i+1})_{1 \leq i < k},\\
(l_i \vee {\bar l}_{i+1}, m_{i+1})_{1 \leq i <k},\\
({\bar l}_k, u_{k+1} - m_{k+1}),\\
(\Box, m_{k+1}) 
\end{array}
\right\}
$$

\noindent where, for 1$\leq$$i$$\leq$$k$+1, $u_i$ is the weight of the corresponding clause, $m_i$=min($u_1, u_2, \ldots, u_i$), and all variables in the literals are different. The weight of a mandatory clause is denoted by~$\top$, and the subtraction $-$ is extended so that $\top-u_i$=$\top$. The chain resolution rule is equivalent to Rule~\ref{linear3} if it is applied to unweighted Max-SAT. The main difference between the chain resolution rule and the weighted version of Rule \ref{linear3} presented
in Section \ref{weightedRule} is that the chain resolution shifts a part of the weight from unit clause $(l_1, m_1 - m_{k+1})$, that is derived in the weighted version of Rule \ref{linear3}, to create unit clauses 
$(l_i, m_i - m_{i+1})_{1 < i \leq k}$, $(l_1, m_1 - m_{k+1})$ itself becoming $(l_1, m_1 - m_2)$.

The cycle resolution rule could be stated as follows:

$$
\left\{
\begin{array}{l}
({\bar l}_i \vee l_{i+1}, u_{i})_{1\leq i <k},\\
({\bar l}_1 \vee {\bar l}_k, u_k)
\end{array}
\right\}
 =  
\left\{
\begin{array}{l}
({\bar l}_1 \vee l_i, m_{i-1}-m_i)_{2\leq i \leq k},\\
({\bar l}_i \vee {l}_{i+1}, u_{i}-m_i)_{2\leq i < k},\\
({\bar l}_1 \vee l_i \vee {\bar l}_{i+1}, m_i)_{2 \leq i <k},\\
(l_1 \vee {\bar l}_i \vee l_{i+1}, m_i)_{2 \leq i <k},\\
({\bar l}_1 \vee {\bar l}_k, u_k-m_k),\\
({\bar l}_1, m_k)
\end{array}
\right\}
$$

When a subset of binary clauses have a cyclic structure, the cycle resolution rule allows to derive a unit clause. Note that the detection of the cyclic structure appears rather time-consuming if it is applied at every node of a search tree and that 2$\times$($k$-2) new ternary clauses have to be inserted. So, \citeauthor{Larrosa06} apply the cycle resolution rule in practice only for the case $k$=3, which is similar to Rule \ref{nonlinear1}, when applied to unweighted Max-SAT. The cycle resolution rule applied to unweighted Max-SAT for $k$=3 can replace Rule~\ref{nonlinear1} and Rule~\ref{nonlinear2} in MaxSatz, but with the following differences compared with Rule~\ref{nonlinear1} and Rule~\ref{nonlinear2}:

\begin{itemize}
\item the application of Rule~\ref{nonlinear1} and Rule~\ref{nonlinear2} is entirely based on inconsistent subformulas detected by unit propagation. The detection of the applicability of Rule~\ref{nonlinear1} and Rule~\ref{nonlinear2} is easy and has very low overhead, since the inconsistent subformulas are always detected in MaxSatz to compute the lower bound (with or without Rule~\ref{nonlinear1}
and Rule~\ref{nonlinear2}).
Every application of Rule~\ref{nonlinear1} or Rule~\ref{nonlinear2} allows to increment the lower bound by 1.
\item the cycle resolution rule needs an extra detection of the cyclic structure, but allows to derive a unit clause from the cyclic structure. The derived unit clause 
could then be used in a unit propagation, and possibly could allow to detect an inconsistent subformula and increase the lower bound by 1.
\end{itemize}

 It would be an interesting future research topic to implement the cycle resolution rule in MaxSat1234 (i.e., MaxSatz without Rule \ref{nonlinear1} and Rule \ref{nonlinear2}) 
to evaluate  the overhead of detecting the cyclic structure and the usefulness of the unit clauses and the ternary clauses derived using the cycle resolution rule, 
and to compare the implemented solver  with MaxSatz. It would be also interesting to compare the chain resolution rule and the cycle resolution rule with the weighted 
inference rules presented in Section~\ref{weightedRule}. 

A more general Max-SAT resolution rule, where the conclusions were not in clausal form, was defined by Larrosa and Heras~\citeyear{LH05a}. Independently, Bonet et al.~\citeyear{BLM06,BLM07}
and Heras and Larrosa~\citeyear{HL06} defined a version of the rule with the conclusions in clausal form. Bonet et al.~\citeyear{BLM06,BLM07} also proved that this rule is complete for Max-SAT. Recently, Ans\'otegui et al.~\citeyear{ABLM07a,ABLM07b} have shown that Max-SAT resolution for many-valued CNF formulas provides a logical framework for
the global and local consistency properties defined for WCN.

\section{Conclusions and Future Work} \label{conclusions}

One of the main drawbacks of state-of-the-art Max-SAT solvers is the lack of suitable inference techniques
that allow to detect as much contradictions as possible and to simplify the formula at each node of the search tree. Existing approaches
put the emphasis on computing underestimations of good quality, but the problem with underestimations is that
the same contradictions are computed once and again. Furthermore, it turns out that $UP$, one of the currently best performing underestimations consisting of detecting disjoint inconsistent subsets of clauses in a CNF formula via unit propagation, is still too conservative. To make  the computation of lowers more incremental and to improve the underestimation, we have defined a number of 
original inference rules for Max-SAT that, based on derived contradictions by unit propagation, transform a Max-SAT instance into an equivalent Max-SAT instance which is easier to solve.
The rules were carefully selected taking into account that they should be applied efficiently. Since all these rules are based on contradiction detection, they should be particularly useful for hard Max-SAT 
instances containing many contradictions.

With the aim of finding out how powerful the inference rules are in practice, we have developed a new Max-SAT solver, called
MaxSatz, which incorporates those rules, and performed an experimental investigation. The results of 
comparing MaxSatz with inference rules and MaxSatz without inference rules provide empirical evidence 
of the usefulness of these rules in making lower bound computation more incremental and in improving the 
quality of lower bounds. The results of comparing MaxSatz with a large selection of the solvers 
available at the time of submitting this paper
provide empirical evidence that MaxSatz, at least for the instances solved, is faster than the other solvers.
We observed gains of several orders of magnitude for the hardest instances. Interestingly, for the benchmarks used, the second best solver
was generally different: UP for Max-2SAT, toolbar for Max-3SAT, MaxSolver for Max-Cut, and MaxSolver and UP for graph 3-coloring. So, MaxSatz is more robust than the
rest of solvers. It is worth mentioning that MaxSatz, enhanced with a lower bound based on failed literal detection~\cite{LMP06},
was the best performing solver for unweighted Max-SAT instances in the Max-SAT Evaluation 2006.
The second and third best performing solvers were, respectively, improved versions of  toolbar and Lazy\footnote{See http://www.iiia.csic.es/\~{}maxsat06 for details. Note that the results of the Max-SAT Evaluation 2006 can be compared with the results of this paper because
they were obtained with the same cluster under the same conditions.}.

As future work we plan to study the orderings of unit clauses in unit propagation to maximize the application of inference rules,
and to define new inference rules for ternary clauses. We are extending the results of this paper to weighted Max-SAT, which is more suitable 
for modeling problems such as maximum clique, set covering and combinatorial auctions, as well as constraint satisfaction problems
such as hard instances of Model RB~\cite{XBHL05,XL06}. We are also adapting the results of this paper to the partial Max-SAT solvers developed by
Argelich and Many\`a~\citeyear{AM05,AM06a,AM07a}.

\section*{Acknowledgments}
Research  partially supported by projects TIN2004-07933-C03-03 and TIN2006-15662-C02-02 funded by the {\em Ministerio de Educaci\'on y Ciencia}. 
The first author was partially supported by National 973 Program of China under Grant No. 2005CB321900.
The second author was supported by a grant {\em Ram\'on y Cajal}. Finally, we would like to thank the referees for their detailed comments and suggestions.

\bibliography{newmy-bib,newothers}

\end{document}